\documentclass[conference]{IEEEtran}
\IEEEoverridecommandlockouts
\usepackage{cite}
\usepackage{amsmath,amssymb,amsfonts}
\usepackage{algorithmic}
\usepackage{graphicx}
\usepackage{textcomp}
\usepackage{xcolor}

\usepackage{graphicx}
\graphicspath{ {figures/} }

\usepackage{hyperref}
\usepackage{amsmath,amsfonts,amsthm}
\usepackage[algo2e,ruled,lined,boxed,commentsnumbered, noend, linesnumbered, procnumbered]{algorithm2e}
\usepackage{cleveref}
\usepackage{siunitx}
\usepackage{booktabs} 
\usepackage{multirow}
\usepackage{tikz}
\usepackage{diagbox}
\usepackage{pifont} 
\usepackage{paralist}
\usepackage{enumitem}
\usepackage{ifthen} 

\usepackage[abbreviations]{foreign}
\usepackage{xcolor}

\Crefname{appendix}{Appendix}{Appendices}

\crefname{assumption}{assump.}{assumptions}
\Crefname{theorem}{Th.}{Th.}
\Crefname{definition}{Def.}{Def.}
\Crefname{lemma}{Lem.}{Lem.}
\Crefname{equation}{Eq.}{Eq.}
\Crefname{section}{Sec.}{Sec.}
\Crefname{figure}{Fig.}{Fig.}
\Crefname{algorithm}{Alg.}{Alg.}
\Crefname{remark}{Rem.}{Rem.}
\Crefname{corollary}{Corr.}{Corr.}
\Crefname{table}{Tab.}{Tab.}

\usepackage{bbm}

\usepackage{float}
\restylefloat{table}

\usepackage{xspace}

\newcommand{\iid}{\ac{IID}\xspace}
\newcommand{\niid}{non-IID\xspace}
\newcommand{\cifar}{CIFAR-10\xspace}
\newcommand{\topk}{\textsc{TopK}\xspace}

\newcommand{\dsgd}{\ac{D-PSGD}\xspace}
\newcommand{\movielens}{MovieLens\xspace}
\newcommand{\celeba}{CelebA\xspace}
\newcommand{\qasc}{QASC\xspace}
\newcommand{\xsum}{XSum\xspace}
\newcommand{\tinytransformer}{T5-Efficient-TINY\xspace}

\newcommand{\rougel}{ROUGE-L\xspace}

\newcommand{\sys}{\ac{CESAR}\xspace}

\newcommand{\codeurl}{\url{https://github.com/sacs-epfl/cesar}} %
\usepackage{tikz}
\usepackage{pgfplots}
\usepackage{pgfplotstable}
\usepackage[eulergreek]{sansmath}
\usepackage{comment}
\newcolumntype{L}[1]{>{\raggedright\let\newline\\\arraybackslash\hspace{0pt}}m{#1}}
\newcolumntype{C}[1]{>{\centering\let\newline\\\arraybackslash\hspace{0pt}}m{#1}}
\newcolumntype{R}[1]{>{\raggedleft\let\newline\\\arraybackslash\hspace{0pt}}m{#1}}

\pgfplotsset{compat=newest}
\usepgfplotslibrary{external,units,colorbrewer,groupplots,fillbetween}
\tikzsetexternalprefix{figures/}
\tikzset{external/mode=list and make}
\usetikzlibrary{patterns}
\usetikzlibrary{shapes.geometric, arrows, positioning}

\def\homepath{}
\makeatletter
\ifdefined\pdfshellescape\ifnum\pdfshellescape>0
\begingroup\endlinechar=-1\relax
\everyeof{\noexpand}%
\edef\x{\endgroup\def\noexpand\homepath{%
		\@@input|"kpsewhich --var-value=HOME" }}\x
\fi\fi
\makeatother

\def\overleafhome{/tmp}
\newcommand{\inputplot}[2]{%
	\ifx\homepath\overleafhome%
	\IfBeginWith{#1}{plots}{\includegraphics{main-figure#2.pdf}}{#1}%
	\else%
	{\sffamily\scriptsize\input{#1}}
	\fi
}

\newcommand{\newgroupwidth}[2]%
{\expandafter\xdef\csname groupwidth#1\endcsname{#2}}

\newcounter{groupwidth}
\newsavebox{\groupwidthbox}
\makeatletter
{\edef\groupnumber{#1}%
	\stepcounter{groupwidth}%
	\@ifundefined{groupwidth\thegroupwidth}{\pgfmathsetlengthmacro{\mywidth}{\linewidth/\groupnumber}}%
	{\expandafter\let\expandafter\mywidth\csname groupwidth\thegroupwidth\endcsname}%
	\begin{lrbox}{\groupwidthbox}%
		\tikzset{/pgfplots/width={\mywidth}}%
		\ignorespaces}%
	{\end{lrbox}%
	\usebox\groupwidthbox
	\pgfmathsetlengthmacro{\mywidth}{\mywidth + (\linewidth - \wd\groupwidthbox)/\groupnumber}
	\immediate\write\@auxout{\string\newgroupwidth{\thegroupwidth}{\mywidth}}}
\makeatother
\usepackage{hyperref} 
\usepackage{acronym}

\acrodef{ML}{machine learning}
\acrodef{RMW}{random model walk}
\acrodef{D-PSGD}{decentralized parallel stochastic gradient descent}
\acrodef{FL}{federated learning}
\acrodef{SGD}{stochastic gradient descent}
\acrodef{IID}{independent and identically distributed}
\acrodef{non-IID}{non independent and identically distributed}
\acrodef{GDPR}{general data protection regulation}
\acrodef{HIPAA}{health insurance portability and accountability act}
\acrodef{CESAR}{communication-efficient secure aggregation}
\acrodef{NLP}{natural language processing}
\acrodef{HbC}{honest-but-curious}
\acrodef{CML}{collaborative machine learning}
\acrodef{P3}{public pool of prompts}
\acrodef{ROC-AUC}{area under the ROC curve}
\acrodef{DL}{decentralized learning}
\acrodef{MIA}{membership inference attack}

\newboolean{showcomments}
\setboolean{showcomments}{false}
\ifthenelse{\boolean{showcomments}}
{ \newcommand{\mynote}[3]{
		\fbox{\bfseries\sffamily\scriptsize#1}
		{\small$\blacktriangleright$\textsf{\emph{\color{#3}{#2}}}$\blacktriangleleft$}}
	\newcommand{\zzz}[1]{{\setlength{\fboxsep}{2pt}\fcolorbox{black}{yellow}{\textsf{\emph{#1}}}}\xspace}}
{ \newcommand{\mynote}[3]{}
	\newcommand{\zzz}[1]{}}

\newcommand{\milos}[1]{\mynote{Milos}{#1}{purple}}

\newcommand\C[1]\null %
\usepackage{thm-restate}
\usepackage{bbm}

\DeclareMathOperator*{\argmin}{arg\,min}
\newcommand{\leftarrowvert}{\mathrel{\leftarrow\!\!\vert}}

\usepackage{amsmath}
\DeclareMathOperator{\opView}{View}
\DeclareMathOperator{\opViewTwo}{View_2}
\DeclareMathOperator{\opComm}{Comm} %
\makeatletter
\AtBeginDocument{%
	\def\ltx@label#1{\cref@label{#1}}%
	
	\def\label@in@display@noarg#1{\cref@old@label@in@display{#1}}%
	
	\def\label@in@mmeasure@noarg#1{%
		\begingroup
		\measuring@false
		\cref@old@label@in@display{#1}%
		\endgroup
	}%
}
\makeatother

\def\BibTeX{{\rm B\kern-.05em{\sc i\kern-.025em b}\kern-.08em
    T\kern-.1667em\lower.7ex\hbox{E}\kern-.125emX}}

\newcommand{\cmark}{\textbf{\textcolor{green!60!black}{\ding{51}}}}
\newcommand{\xmark}{\textbf{\textcolor{red!80!black}{\ding{55}}}}

\newboolean{isextendedversion}
\setboolean{isextendedversion}{true}
    
\begin{document}

\title{Communication-Efficient Secure Aggregation \\in Decentralized Learning}

\newcommand{\epflblock}{
\textit{EPFL}\\
Lausanne, Switzerland
}

\author{
\IEEEauthorblockN{Sayan Biswas, Anne-Marie Kermarrec, Rafael Pires, Rishi Sharma, and Milos Vujasinovic}
\IEEEauthorblockA{\textit{EPFL}\\
Lausanne, Switzerland\\
\{sayan.biswas, anne-marie.kermarrec, rafael.pires,
rishi.sharma, milos.vujasinovic\}@epfl.ch}
}

\maketitle

\begin{abstract}%

\leavevmode \Ac{DL} enables participants to collaboratively train models without a central server, yet it faces significant scalability challenges that demand sparsification to reduce the prohibitive communication costs of peer-to-peer exchange.
While secure aggregation effectively mitigates privacy risks in standard settings, it has remained fundamentally incompatible with sparsification in decentralized networks due to the mismatch of indices across local updates, forcing a trade-off between communication efficiency and privacy.
This paper introduces CESAR, a novel protocol that resolves this incompatibility by integrating secure aggregation and sparsification to provide provable defense against honest-but-curious and colluding adversaries.
By coordinating masks over parameter intersections, CESAR supports node dropouts and robust privacy without central aggregation.
Empirical evaluations on models up to \num{124} million parameters demonstrate that CESAR matches the accuracy of non-private baselines under equal sparsification while cutting total data exchange by \SI{66.7}{\percent} compared to a standard full-parameter decentralized protocol (D-PSGD).
With TopK sparsification on IID data, CESAR even exceeds by \SI{0.3}{\percent} the accuracy achieved by D-PSGD with sparsification.
Collectively, these results establish CESAR as the first decentralized protocol to achieve both privacy and communication efficiency through secure aggregation in \ac{DL}.

\end{abstract} 

\begingroup
\renewcommand\thefootnote{}%
\makeatletter\renewcommand\@makefntext[1]{\noindent #1}\makeatother
\footnotetext{%
This work has been funded by the Swiss National Science Foundation, under the project ``FRIDAY: Frugal, Privacy-Aware and Practical Decentralized Learning'', SNSF proposal No. 10.001.796.
The authors thank Sami Abuzakuk, Akash Dhasade, and Martijn de Vos for constructive discussions, and the anonymous reviewers of SRDS 2026 for their valuable feedback and time. %
}
\endgroup

\acresetall
\section{Introduction}

The performance of modern \ac{ML} models relies heavily on large, high-quality datasets~\cite{ramezan_effects_2021}, creating a strong incentive for collaborative learning to pool diverse data from multiple sources~\cite{sheller2018multiinstitutionaldeeplearningmodeling}. 
However, direct data sharing poses significant privacy and security challenges, particularly in sensitive domains governed by regulations like \ac{GDPR} and \ac{HIPAA}. 
\Ac{DL}~\cite{firstdlpaper} addresses these concerns by keeping data distributed across the network, allowing learning to proceed through the exchange of local model parameters rather than raw data. 
Although similar approaches have been deployed in centralized settings for applications like text prediction~\cite{keyboardImprovementFL} and healthcare~\cite{radiationoncologyFL}, \ac{DL} remains essential when a coordinating entity is infeasible or undesirable. 
By eliminating the central server to avoid a single point of failure and authority, \ac{DL} empowers local control and enables distributed decision-making. 
Yet, despite its growing popularity, this architecture introduces two pressing challenges: communication efficiency and privacy.

\textbf{Communication Efficiency.} 
Standard \ac{DL} requires nodes to exchange high-dimensional model parameters with all neighbors every round~\cite{lian2017dpsgd}. 
As models grow in size, this exchange becomes prohibitively expensive.
Sparsification techniques mitigate this cost by transmitting only a selected subset of parameters~\cite{alistarh2018sparseconvergence,tang2020sparsification}, dramatically reducing communication while largely preserving model accuracy~\cite{lin2020deepgradientcompressionreducing}.

\textbf{Privacy.} 
Even without sharing raw data, exchanged model updates can leak sensitive information about local datasets~\cite{shokriMembershipInferenceAttacks2017,carlini2022membership}, preventing standalone \ac{DL} deployment in privacy-critical domains. 
Existing privacy-preserving approaches for \ac{DL} focus primarily on noise-based~\cite{cyffers2022muffliato,biswas2025low} and randomized model sharing techniques~\cite{biswas2025noiseless}, which degrade utility and slow convergence. 
By contrast, secure aggregation~\cite{bonawitz2017practical}, widely adopted in settings with a centralized aggregator, offers a noise-free alternative: each participant masks its update with coordinated random values that cancel upon summation, revealing only the aggregate. 
In \ac{ML}, secure aggregation primarily targets single, large-scale aggregations involving many participants. 
It remains underexplored in fully decentralized networks where training proceeds through numerous smaller aggregations among few peers each round, and where no central aggregator exists. 
Adapting secure aggregation to this setting introduces distinct constraints that make protocol design challenging. \Cref{tab:approach_comparison} summarizes various approaches for privacy-preserving and communication-efficient \ac{DL}.

\begin{table}[!ht]
	\centering
	\begin{tabular}{p{4.1cm}|ccc}
		\toprule
		\textbf{Approach} & \textbf{Comm.\ eff.} & \textbf{Privacy}\\%
		\midrule
		Standard DL~\cite{lian2017dpsgd} & \xmark & \xmark\\%
		Standard DL + sparsification & \cmark & \xmark\\ %
		Noise-based~\cite{cyffers2022muffliato,biswas2025low} & \xmark & \cmark\\ %
		Secure aggregation & \xmark & \cmark\\ %
		\textbf{CESAR (ours)} & \cmark & \cmark\\ %
		\bottomrule
	\end{tabular}
	\caption{Comparison of training approaches in \ac{DL}.}
	\label{tab:approach_comparison}
\end{table}

While sparsification and secure aggregation individually address communication and privacy, combining them in \ac{DL} introduces a critical conflict.
Secure aggregation requires each shared parameter to be masked and transmitted by multiple nodes.
Yet, sparsification techniques such as \topk~\cite{alistarh2018sparseconvergence} and random subsampling~\cite{tang2020sparsification} select different parameter sets across participants.
Consequently, a parameter shared by only one participant creates a dilemma: it must be shared unmasked, aggregated without mask cancellation, or dropped.
Each option compromises either privacy or convergence, making this interplay the key obstacle to achieving both communication efficiency and strong privacy in \ac{DL}.

This paper addresses this gap by introducing \acs{CESAR} (\emph{\aclu{CESAR}}), a secure aggregation protocol for fully decentralized \ac{ML}.
\acs{CESAR} adapts secure aggregation techniques from centralized frameworks for use with \ac{D-PSGD}, the canonical non-private algorithm for \ac{DL}, while natively supporting sparsification.
Prior work on secure aggregation with sparsification targets \ac{FL}: SparseSecAgg~\cite{ergun2021sparsified} integrates sparsification directly into the aggregation protocol but limits compatibility with other sparsification techniques, while Lu~\etal~\cite{lu2023top} incorporate \topk by masking the union of indices selected across nodes.
Both approaches rely on a central server and cannot be trivially adapted to \ac{DL}.
By contrast, \sys operates in a fully serverless setting, coordinates masks over two-hop neighborhoods, and resolves index mismatches by masking only the intersection of locally selected indices.
The protocol defends against \ac{HbC} adversaries, resists collusion through a configurable masking requirement, and gracefully handles node dropouts with minimal communication overhead and minor accuracy impact.
Experiments on models up to \num{124}M parameters demonstrate that \sys matches \ac{D-PSGD} accuracy while reducing data exchange by \SI{66.7}{\percent} compared to full-parameter sharing, and with \topk on IID data, exceeds \ac{D-PSGD} accuracy by 0.3\%.

\textbf{Contributions.} 
We make the following contributions:
\begin{itemize}
	\item We introduce \sys (\Cref{sec:cesar_explained}), the first secure aggregation protocol for \ac{DL} that integrates sparsification and dropout handling without requiring additional servers.
	\item We formally prove security against \ac{HbC} adversaries and provide theoretical analysis enabling provable resilience against bounded collusion (\Cref{sec:theorethical_analysis}).
	\item We derive the expected fraction of shared parameters under colluding adversaries (\Cref{sec:theorethical_analysis}), enabling principled parameterization of \sys.
	\item We evaluate \sys across five datasets and models up to 124M parameters (\Cref{sec:eval_different_datasets}), demonstrating accuracy matching \dsgd with 7--35\% communication overhead. For fixed communication budgets, \sys achieves higher accuracy than both full-parameter secure aggregation and unsparsified \dsgd (\Cref{sec:full_secagg_comparison}).
	\item We show that \sys with dropout support loses only 0.36\% accuracy at 10\% dropout rate while reducing total data exchange (\Cref{sec:exp:dropout_handling}).
\end{itemize}

\section{Preliminaries}
\label{sec:preliminaries}

\subsection{Decentralized Learning}

In \ac{DL} nodes $\mathcal{N}$ are organized in a network topology $G$. Each node $N_i \in \mathcal{N}$ holds a local dataset $D_i$ with distribution $\varphi_i$. These datasets remain private throughout training. If local label distributions are homogeneous across nodes, datasets are termed \iid, and \niid~otherwise. The collective objective is to find optimal parameters $\theta^*$ minimizing a loss function $\mathcal{L}$ across the aggregated dataset, \ie $\theta^* = \argmin_\theta \mathcal{L}(D; \theta)$ s.t. $D = \bigcup_{N_j \in \mathcal{N}} D_j$.

Training proceeds over multiple rounds, each comprising three steps: training, sharing, and aggregation. First, each node performs one or more iterations of \ac{SGD} on mini-batches of local data. Next, nodes share updated parameters with all neighbors, as in \ac{D-PSGD}~\cite{lian2017dpsgd}.
Finally, each node averages received parameters with its own model, producing the starting point for the next round. This process iterates until convergence.

\subsection{Sparsification}

As models grow increasingly large, transferring them over a network incurs significant costs. Sparsification addresses this through lossy compression, transmitting only a selected subset of parameter indices rather than the entire model. %

Two popular approaches are random subsampling~\cite{tang2020sparsification} and \topk~\cite{alistarh2018sparseconvergence}. Both take as input a selection fraction $\alpha$ and a model of size $d$. Random subsampling selects each index independently with probability $\alpha$, expecting $\alpha d$ indices. \topk sorts parameters by magnitude and selects the top $\alpha d$ values.

\subsection{Secure Aggregation}
\label{sec:secagg}

Secure aggregation~\cite{bonawitz2017practical} is a multiparty computation technique allowing mutually distrustful parties to jointly compute a linear function of their private inputs, canonically the sum. Given $n \ge 3$ nodes $\mathcal{N} = \{N_1, \ldots, N_n\}$, each node $N_i$ holds a private value $v_i$ and seeks to compute $f(v_1, \ldots, v_n) = \sum_{i=1}^n v_i =: \hat{v}$ such that individual values remain private while the aggregate $\hat{v}$ becomes public.

Pairwise masking~\cite{acs2011dream, bonawitz2017practical} achieves this by imposing an ordering $N_1 < \ldots < N_n$. Each pair $N_i, N_j$ (s.t. $N_i < N_j$) generates a private random mask $M_{i,j}$. Node $N_i$ adds $M_{i,j}$ to its value while $N_j$ subtracts it, ensuring masks sum to zero across all pairs. Each node $N_i$ computes a masked value $v^{'}_i = v_i - \sum_{j=1}^{i-1} M_{j,i} + \sum_{j=i+1}^n M_{i,j}$. This $v^{'}_i$ can be shared publicly, revealing no information about the original $v_i$. 
\section{\sys: Secure aggregation for sparsified decentralized learning}
\label{sec:cesar}

\subsection{System model} 
\label{sec:system_model}

We consider a synchronous network of $n$ nodes $\mathcal{N} = \{N_1, \dots, N_n\}$, interconnected in a communication graph with bidirectional, reliable, and encrypted channels (\eg, TLS over TCP).
No central server exists: all nodes are equal in hierarchy.
A peer sampling service~\cite{jelasity2007gossip,guerraoui2024peerswap} generates an \emph{aggregation graph} $G(\mathcal{N}, E)$ as an overlay network of the communication graph; all subsequent references to \emph{graph} denote this overlay.

\sys requires each node to communicate with immediate and second-degree neighbors.
We denote immediate neighbors of $N \in \mathcal{N}$ as $\operatorname{View}(N)=\{M \in \mathcal{N} : \{N, M\} \in E\}$ and second-degree neighbors as $\operatorname{View}_2(N)=(\cup_{M\in\operatorname{View}(N)} \operatorname{View}(M)) \setminus \{N\}$.
For $N_{r},N_i\in\mathcal{N}$ with $\{N_{r},N_i\}\in E$, let $\operatorname{Comm}(N_{r}, N_i)=\operatorname{View}(N_{r})\setminus \{N_{i}\}$ denote nodes that $N_i$ coordinates with to determine index intersections and agree on masks before sharing masked parameters to \emph{receiving node} $N_{r}$.

All nodes train models with shared architecture of $d$ parameters enumerated by indices $\mathcal{I} = \{1, \ldots, d\}$.
Local parameters of $N_i$ are $v_i \in \mathbb{R}^d$, where $v_i^{(p)}$ denotes the value at index $p \in \mathcal{I}$.
The indices selected by $N_i$ for sharing are $I_{i} \subseteq \mathcal{I}$.%

\textbf{Pseudorandom Mask Generation.} Letting $\Sigma$ denote the space of all seeds, we assume all nodes have access to a pseudorandom number generator $G_\sigma: \mathbb{N} \rightarrow \mathbb{Z}_q$. For a given index $i$, the function $G_\sigma(i)$ returns the $i$-th element of the pseudorandom sequence determined by $\sigma \in \Sigma$. With $\mathcal{P}(\mathcal{X})$ representing the power set of any set $\mathcal{X}$, we define a \emph{random mask function} $\mathsf{RM}_d: \Sigma \times \Sigma \times \mathcal{P}([d]) \rightarrow \mathbb{Z}_q^d$. This function generates a $d$-dimensional vector where coordinates \emph{not} in $I \subseteq [d]$ are zero. We require this function to satisfy the antisymmetry requirement, $\mathsf{RM}_d(\sigma_i, \sigma_j, I) = -\mathsf{RM}_d(\sigma_j, \sigma_i, I)$, ensuring that pairwise masks cancel upon summation. For the $k$-th component of the mask, we define:
\[
[\mathsf{RM}_d(\sigma_i, \sigma_j, I)]_k = 
\begin{cases}
  G_{\sigma_i}(k) - G_{\sigma_j}(k), & \text{if } k \in I, \\
  0, & \text{otherwise}.
\end{cases}
\]

\subsection{Threat model}\label{sec:threat}
\sys targets permissioned networks with verifiable identities, such as cross-silo collaborative learning.
We focus on adversaries who adhere to the protocol, active manipulations are out of scope.

Our primary threat model assumes passive \acf{HbC} adversaries seeking to extract raw parameters from victim nodes. %
Any node may act as an \ac{HbC} adversary using all available information to infer other nodes' models. This threat model is consistent with the standard \ac{DL} literature~\cite{cyffers2022muffliato,biswas2025noiseless,biswas2025low}. We consider \sys in settings both with and without collusion among the adversaries.
For the former, we assume the existence of a single colluding set of size $\geq 2$ where all retrieved private information is shared among members.

Privacy concerns in this work pertain only to local models.
Because secure aggregation does not modify the aggregated model, protecting it falls outside the scope of secure aggregation: noise-based solutions (orthogonal to this work) may address this at reduced utility.
We also consider leakage from sparsification metadata out of scope, as such metadata is method-specific.
However, sparsification provides privacy benefits over full model sharing (demonstrated in \Cref{sec:exp:partial_leakage_eval}), which we argue outweigh potential metadata leakage.

\subsection{\sys algorithm}
\label{sec:cesar_explained}

\noindent\textbf{Intuition and Key Challenges.}
Standard secure aggregation relies on pairwise masks that cancel upon summation: if node $N_a$ holding $v_a$ adds a mask $M$ and node $N_b$ holding $v_b$ subtracts $M$, then any party that receives and sums \emph{both} contributions observes only $v_a + v_b$.
In centralized settings, this condition is trivially satisfied because a single server receives all masked updates and performs the aggregation.

In \ac{DL}, however, aggregation is local.
Each node $N$ aggregates updates only from its neighbors $\opView(N)$, and different nodes generally have different neighborhoods.
If $N_a$ masks its update with $M_{a,b}$ (agreed with $N_b$), the corresponding $-M_{a,b}$ must be included in the \emph{same aggregation} for cancellation to occur.
That is, every node receiving a masked update from $N_a$ must also receive the counter-masked update from $N_b$, a condition not guaranteed in general decentralized graphs.

\sys addresses this by generating \emph{recipient-specific} masks.
When node $N_i$ sends an update to a neighbor $N_r$, it coordinates with all other neighbors of $N_r$, $\opComm(N_r,N_i)$, to establish pairwise masks.
As a result, all nodes contributing to $N_r$ apply masks that cancel when $N_r$ performs aggregation, regardless of differences in their local neighborhoods.

Sparsification introduces a second challenge:
independent selection causes nodes $N_i$ and $N_j$ to choose differing index sets $I_i$ and $I_j$.
For indices in the symmetric difference $I_i \triangle I_j$, nodes cannot form cancelling mask pairs, resulting in unmasked transmissions or incomplete cancellation.
\sys addresses this by masking only the shared indices $I_i \cap I_j$ and discarding parameters without sufficient masks, ensuring both cancellation and privacy.

\begin{figure*}[ht]
	\centering
	\includegraphics[width=0.75\textwidth]{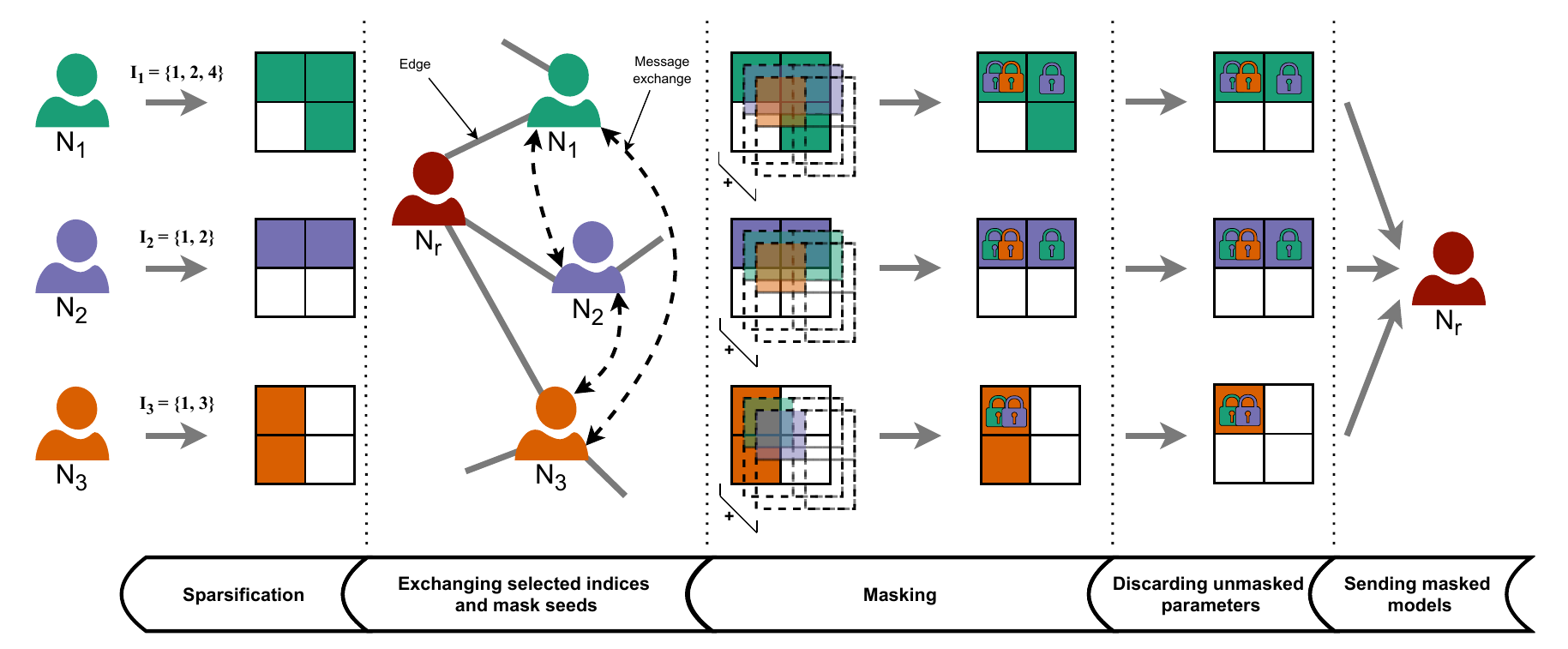}
	\caption{Overview of \sys with nodes $N_1$, $N_2$, and $N_3$ sending to common neighbor $N_r$ for $s=1$. 
        (1) Each node independently sparsifies its local model; 
        (2) Nodes exchange selected indices and partial seeds with mask partners; 
        (3) Each node masks indices in the pairwise intersections; 
        (4) Indices with insufficient masks are discarded; 
        (5) Masked sparse models are sent to $N_r$.}
	\label{fig:cesar_overview}
\end{figure*}

\smallskip
\noindent\textbf{Protocol Walkthrough.}
\label{sec:worked_example}
\Cref{fig:cesar_overview} illustrates the execution of \sys, where nodes $N_1$, $N_2$, and $N_3$ prepare masked updates for their common neighbor $N_r$. We narrate from the perspective of $N_1$, though the protocol runs in parallel on all nodes. Each node repeats this process independently for every neighbor, but sparsification is performed only once per round. 
The masking requirement $s$ is a tunable hyperparameter defining the minimum number of pairwise masks a parameter must carry to be shared. We require $s$ to be consistent across the system (formalized in \Cref{corr:samemaskingreq}). Here we set $s=1$.

\textbf{Step 1: Sparsification.}
Each node independently selects indices: $N_1$ selects $I_1 = \{1, 2, 4\}$, $N_2$ selects $I_2 = \{1, 2\}$, and $N_3$ selects $I_3 = \{1, 3\}$.

\textbf{Step 2: Exchanging Indices and Seeds.}
To send to $N_r$, node $N_1$ coordinates with its mask partners $\opComm(N_r, N_1) = \{N_2, N_3\}$. With each partner $N_j$, node $N_1$ exchanges a partial seed $\sigma_{1j}$ and its index set $I_1$, receiving $\sigma_{j1}$ and $I_j$ in return. These index sets determine where masks can be applied: only indices in $I_1 \cap I_j$ can carry a mask agreed with $N_j$.

\textbf{Step 3: Masking.}
Using the exchanged seeds, $N_1$ computes masks via $\text{RM}_d$:
$M_{1,2} = \text{RM}_d(\sigma_{12}, \sigma_{21}, I_1 \cap I_2)$ covering $\{1, 2\}$, and 
$M_{1,3} = \text{RM}_d(\sigma_{13}, \sigma_{31}, I_1 \cap I_3)$ covering $\{1\}$.
By antisymmetry of $\text{RM}_d$, partners $N_2$ and $N_3$ independently compute $M_{2,1} = -M_{1,2}$ and $M_{3,1} = -M_{1,3}$ from the same seeds. Node $N_1$ counts masks per index: index 1 receives two, index 2 receives one, index 4 receives none.

\textbf{Step 4: Discarding Unmasked Parameters.}
Any index with fewer than $s$ masks is discarded. With $s=1$, index 4 fails this threshold and is dropped from $I_1$.

\textbf{Step 5: Sending Masked Model.}
Let $I_{1 \to r} = \{1,2\}$ be the surviving indices. $N_1$ computes masked values $\tilde{v}_{1\to r}^{(p)} = v_1^{(p)} + M_{1,2}^{(p)} + M_{1,3}^{(p)}$ for each $p \in I_{1\to r}$, then transmits $(I_{1 \to r}, \tilde{v}_{1\to r})$ to $N_r$. Upon receiving masked models from all neighbors, $N_r$ aggregates them. Masks cancel since $M_{1,2} + M_{2,1} = 0$ and $M_{1,3} + M_{3,1} = 0$, recovering the true sum of local parameters.

\smallskip
\noindent\textbf{Algorithm Specification.}
\Cref{alg:basicSecAgg} formalizes \sys.
The notation $v[I]$ extracts the subvector of $v$ at indices $I$.

\begin{algorithm2e}[t!]
    \DontPrintSemicolon
    \caption{CESAR on node $N_i$}
    \label{alg:basicSecAgg}
    \small
    \KwIn{Model $v_i \in \mathbb{R}^d$, selected indices $I_i$, masking requirement $s$}
    \KwOut{Aggregated model $v'_i$}
    
    \BlankLine
    \For{$N_j \in \opViewTwo(N_i)$}{
        Send $(\sigma_{ij}, I_i)$ to $N_j$; receive $(\sigma_{ji}, I_j)$\;
        $M_{i,j} \gets \mathrm{RM}_d(\sigma_{ij}, \sigma_{ji}, I_i \cap I_j)$\;
    }
    \For{\textcolor{blue}{$N_j \in \opView(N_i) \setminus \opViewTwo(N_i)$}}{
        \textcolor{blue}{Send $I_i$ to $N_j$, receive $I_j$}\;
    }
    
    \BlankLine
    \For{$N_r \in \opView(N_i)$}{
        $\mathcal{P} \gets \opComm(N_r, N_i)$ \textcolor{blue}{$\setminus$ crashed}\;
        $I_{i \to r} \gets \emptyset$\;
        \For{$p \in I_i$}{
            $c_p \gets$ number of nodes in $\mathcal{P}$ that selected $p$\;
            \lIf{$c_p \geq s$}{add $p$ to $I_{i \to r}$}
        }
        $\tilde{v}_{i \to r} \gets (v_i + \sum_{N_j \in \mathcal{P}} M_{i,j})[I_{i \to r}]$\;
        Send $(I_{i \to r}, \tilde{v}_{i \to r})$ to $N_r$; receive $(I_{r \to i}, \tilde{v}_{r \to i})$\;
    }
    
    \BlankLine
    \For{$p \in [d]$}{
        \eIf{\textcolor{blue}{a crashed neighbor had selected $p$}}{
            \textcolor{blue}{$v'^{(p)}_i \gets v^{(p)}_i$}\;
        }{
            $v'^{(p)}_i \gets$ mean of $v^{(p)}_i$ and received values at $p$\;
        }
    }
    \Return $v'_i$
\end{algorithm2e}

\smallskip
\noindent\textbf{Handling Dropouts.}
\label{sec:handling_dropouts}
In real-world deployments, nodes may disconnect, making robust dropout handling essential.
A straightforward approach is as follows: if a receiving node $N_r$ detects that a neighbor has crashed before delivering its masked model, $N_r$ aborts the current round and reuses its local model for the next round.

A more refined strategy is possible. 
Rather than aborting entirely, $N_r$ can identify and discard only those parameters that the failed node would have contributed, integrating partial contributions from remaining neighbors and preserving most of the aggregation round.

We refer to this method as \emph{parameter-wise dropout handling}, integrated into \sys as highlighted in \textcolor{blue}{blue} in \Cref{alg:basicSecAgg}. 
To enable this, each node must share its selected indices not only with second-degree neighbors but also with direct neighbors during mask coordination.
Detection itself relies on timeouts: under the synchronous model of \Cref{sec:system_model}, these realize a perfect failure detector~\cite{chandra1996unreliable}.
                                                                
\subsection{Theoretical analysis}
\label{sec:theorethical_analysis}

\ifthenelse{\boolean{isextendedversion}}{}{
    In the interest of space, all proofs are deferred to the extended version of this paper \cite{cesar-extended}.
}

\subsubsection{Privacy guarantees}

In this section we study the privacy guarantees provided by \sys. 
We strictly consider the privacy of parameter values of local models.
Consequently, any information leakage from the final aggregated model or leakage caused by sparsification is beyond the scope of this theoretical analysis.
Empirical leakage of aggregated models and sparsification is evaluated in \Cref{sec:exp:partial_leakage_eval}.

We consider a parameter private if its exact value cannot be inferred by any other node in the network.
A node is considered private if all its parameters are private. In the remainder of this section, we derive a lower bound on the worst-case privacy of \sys, considering secure aggregation and sparsification. 

\begin{restatable}{theorem}{hbcresilient}\label{th:hbc_resilient}
\sys ensures that no \ac{HbC} adversary can infer the exact value of any parameter received from its neighbors.
\end{restatable}
\ifthenelse{\boolean{isextendedversion}}{
    \begin{proof}
       Postponed to \Cref{app:proof:hbc_resilient}.
    \end{proof}
}{}

\subsubsection{Collusion}
\label{sec:collusion_theory}

\ac{DL} with colluding adversaries poses significant challenges. 
Here, we demonstrate how \sys can be adapted to resist such collusion.
In \Cref{sec:cesar_explained}, we introduced \Cref{alg:basicSecAgg} and explained its operation with a masking requirement of $s=1$. 
Increasing $s$ strengthens collusion resistance in a precise sense: by \Cref{th:collusions}, a masking requirement of $s$ prevents any coalition of at most $s$ colluding adversaries from inferring the exact value of any parameter of the local model of a trusted node.
However, increasing $s$ raises concerns about the potential for some masks to be discarded while their corresponding opposite masks are not. Below%
, we show how the correctness of \sys remains unaffected by this. 

Let us consider a setting where $n_t$ trusted nodes honestly comply with \sys and $k$ adversarial nodes collude to share all the information they have with each other aiming to violate the privacy of the trusted nodes. Let $\mathcal{N}_T$ and $\mathcal{N}_A$, respectively, denote the sets of the trusted and the colluding adversarial nodes partitioning $\mathcal{N}$. In this setting, studying the privacy guarantees stays relevant only for the trusted nodes. In the subsequent analysis, we make the following assumption.

\begin{restatable}{assumption}{samemaskingreq}\label{corr:samemaskingreq}
    All trusted nodes have the same masking requirement when communicating with a fixed neighbor.   
\end{restatable}

\begin{restatable}{theorem}{numberofmasks}\label{lem:numberofmasks}

    For a fixed $p \in I$ and receiving node $N_j$, all $N_k\in \operatorname{View}(N_j)$, if $p\in I_{k}$, apply the same number of masks to the parameter value at index $p$ before discarding. 
\end{restatable}
\ifthenelse{\boolean{isextendedversion}}{
    \begin{proof}
       Postponed to \Cref{app:proof:numberofmasks}.
    \end{proof}
}{}

\milos{I added \Cref{lem:graphreduction} back again. (we had it ommitted for ICDCS since I believe due to exclusion of proofs we didn't see necessary keeping this considering it's only a building block)}
\begin{restatable}{lemma}{graphreduction}\label{lem:graphreduction}
  If the masking requirement of the nodes is $k$, without worsening the privacy of the trusted nodes, any graph under CESAR, where $|\mathcal{N}_A| \le k$, can be transformed into a graph where: 
\begin{enumerate}
    \item[i.] $|\mathcal{N}_A|=k$
    \item[ii.] All trusted nodes have a degree of at least 1
    \item[iii.] the nodes in $\mathcal{N}_A$ form a clique
    \item[iv.] $\{N_i, N_j\} \notin E$ for any $N_i, N_j \in \mathcal{N}_T$
\end{enumerate}
\end{restatable}
\ifthenelse{\boolean{isextendedversion}}{
    \begin{proof}
       Postponed to \Cref{app:proof:graphreduction}.
    \end{proof}
}{}

\Cref{lem:numberofmasks} demonstrates that all neighbors of a node, which have selected the parameter at a predetermined position for transmission to the given node, have an identical number of masks applied to the parameter. 
Under \Cref{corr:samemaskingreq}, it is guaranteed that all such neighbors will either transmit the parameter, thereby cancelling out the masks through aggregation, or, \emph{all} of them will discard the parameter.
This ensures that all masks cancel out and the correctness of \sys remains unaffected. 

\begin{restatable}{theorem}{collusions}\label{th:collusions}
If $k$ is the masking requirement of every trusted node in the network, the true values of their model parameters will remain obfuscated from every node in the network when there are $k$ colluding adversarial nodes.   
\end{restatable}

\ifthenelse{\boolean{isextendedversion}}{
    \begin{proof}
       Postponed to \Cref{app:proof:collusions}.
    \end{proof}
}{}

As \sys discards unmasked indices, thus reducing the percentage of model parameters intended for sharing in sparsification, understanding how this varies based on system parameters, masking requirements, and training configuration is vital to evaluate the practical application of \sys. 

Given the difficulty of exhaustively modeling every sparsification approach, we focus on a simplified yet realistic scenario. 
Here, each index is equally likely to be selected in sparsification, and the selection is mutually independent between nodes.
While this assumption does not hold for many sparsification methods, it perfectly characterizes random subsampling. 
We assume that every node chooses $\alpha$ fraction of their indices in sparsification. Since all nodes select the same fraction and the overall model sizes are identical, the probability $\pi_i^{(p)}$ for any node $i$ to select a given index $p$ is $\alpha$. This enable derivation of the expected number of model parameters shared by each node in the network under \sys capturing an environment when we have $k$ colluding nodes.  

\begin{restatable}{theorem}{expectedparamaterssharedwithcollusion}\label{th:expected_parameters_shared_collusion}
    \sys with random subsampling, where nodes have a masking requirement of $s$, if $\beta(\alpha, \delta, s)$ is the expected proportion of parameters shared by one node to another in a round, we have $\beta\left(\alpha,\delta,s\right)=\sum_{i=s}^{\delta-1} \binom{\delta-1}{i} \alpha^{i+1} (1-\alpha)^{\delta-1-i}$,
    where $\alpha$ is the probability of independently sampling any parameter by any node and $\delta=\left\lvert\operatorname{View}(\text{Receiver})\right\rvert$. 
\end{restatable}

\ifthenelse{\boolean{isextendedversion}}{
    \begin{proof}
       Postponed to \Cref{app:proof:expected_parameters_shared_collusion}.
    \end{proof}
}{}

\begin{restatable}{corollary}{expectedparamatersshared}\label{th:expected_parameters_shared}
    In the absence of colluding nodes, \sys with random subsampling  ensures that the expected proportion of parameters shared by any node (sender) to another (receiver) in one round is given by $\beta\left(\alpha, \delta, 1\right)=\alpha \left(1-(1-\alpha)^{\delta - 1}\right)$,
    where $\alpha$ is the probability of (independently) sampling any parameter by any node and $\delta=\left\lvert\operatorname{View}(\text{Receiver})\right\rvert$. 
\end{restatable}
\ifthenelse{\boolean{isextendedversion}}{
    \begin{proof}
       Postponed to \Cref{app:proof:expected_parameters_shared}.
    \end{proof}
}{}

By \Cref{th:expected_parameters_shared}, as $\delta$ increases, $\beta$ converges to $\alpha$, \ie, the expected proportion of shared parameters becomes equal to the expected proportion of parameters selected in sparsification. Moreover, this convergence is faster for larger values of $\alpha$.

\subsubsection{Communication overhead}
\label{sec:overhead_analysis}

In this analysis, $\alpha$ fraction of indices of a model of size $d$ is selected in sparsification. The maximum degree of any node in the graph is $\delta_\text{max}$.

The communication overhead incurred by \sys occurs primarily in the prestep stage. 
During this phase, each node dispatches a single message to the neighbors of its neighbors. 
The number of these message transmissions is bounded by $O(\delta_\text{max}^2)$.
Each of these messages comprises a partial seed of $O(1)$, and a set of indices selected in sparsification of size $O(\alpha d)$.
Consequently, the total communication overhead for the prestep in \sys amounts to $O(\alpha d \delta_\text{max}^2)$.
 
\section{Evaluation}
\label{sec:eval}

Our evaluation validates \sys's performance across three critical dimensions: utility, efficiency, and privacy. Specifically, we structure our experiments to answer the following: \textbf{RQ1 (Utility \& Efficiency):} Can CESAR match the accuracy of non-private baselines (D-PSGD) while significantly reducing communication costs? (\Cref{sec:eval_different_datasets,sec:full_secagg_comparison}) \textbf{RQ2 (Robustness):} How does the masking requirement $s$ affect convergence and resilience against collusion? (\Cref{sec:eval_collusion,sec:masking_req_experiments}) \textbf{RQ3 (Privacy):} Does the protocol effectively limit information leakage against inference attacks compared to standard baselines? (\Cref{sec:exp:partial_leakage_eval}) \textbf{RQ4 (Resilience):} Can CESAR maintain model performance in volatile networks with frequent node dropouts? (\Cref{sec:exp:dropout_handling}). We evaluate \sys against \dsgd across multiple graph topologies, sparsification fractions, and five datasets with models ranging from 30K to 124M parameters.
The implementation is available at \codeurl{}. It uses Python 3.6 with the \emph{decentralizepy} library~\cite{decentralizepy}, built on PyTorch.

\subsection{Experimental Setup}

\textbf{Datasets and Hyperparameters.}
Experiments span \cifar, \celeba, \qasc, \movielens, and \xsum.
\qasc prompts use \ac{P3}~\cite{bach2022promptsource}.
These cover image classification (\cifar, \celeba), question answering (\qasc), recommendation (\movielens), and text summarization (\xsum) tasks.

Training data is partitioned so no two nodes share samples and the full test set is used to evaluate performance at each test step (\xsum uses the first \num{5000} of \num{11334} test samples).
For IID distribution, datasets are shuffled and split into equal chunks per node.
\movielens is inherently non-IID.
For non-IID partitioning performed only on \cifar, labels are sorted and split into $2n$ chunks, with each node receiving 2 chunks (limiting each node to at most 4 of 10 classes).

Training uses SGD without momentum and hyperparameters are tuned via grid search on \dsgd with full sharing.
\Cref{tab:hyperparameters} lists all hyperparameters, dataset sizes, and models.

\begin{table*}[ht]
    \centering
    \caption{Dataset-specific training hyperparameters, properties, and models.}
    \label{tab:hyperparameters}
    \setlength{\tabcolsep}{4pt}
    \begin{tabular}{l|ccccc}
        \toprule
         & \textbf{\cifar} & \textbf{\celeba} & \textbf{\qasc} & \textbf{\movielens} & \textbf{\xsum}\\
        \midrule
        Learning rate & 0.01 & 0.001 & 0.001 & 0.05 & 0.01\\
        Batch size & 8 & 8 & 4 & 64 & 8\\
        Comm. rounds/epoch & 20 & 10 & 2 & 22 & 15\\
        SGD steps/round & 6 & 10 & 31 & 1 & 53\\
        \midrule
        Training samples & 50K & 63.7K & 8.1K & 70K & 204K\\
        Task & Img. class. & Img. class. & QA & Recommend. & Text summ.\\
        Model & 4 Conv2D & GN-LeNet~\cite{hsiehskewscout2020} & GPT2-small~\cite{gpt2paper} & Matrix Fact.~\cite{korenmatrixfactorization2009} & T5-Eff.-TINY~\cite{2020t5}\\
        Parameters & 89.8K & 30.2K & 124.4M & 206.7K & 15.6M\\
        \bottomrule
    \end{tabular}
\end{table*}

\textbf{Hardware and topology.}
We use two hardware configurations.
\cifar, \celeba, and \movielens run on three machines, each with a 32-core Intel Xeon E5-2630 v3 (2.40~GHz). %
These experiments use a 48-node network (16 nodes per machine).
\xsum and \qasc run on a cluster with 128 AMD EPYC 7543 cores (2.80~GHz), one NVIDIA A100-80GB (\xsum) or three NVIDIA H100-80GB (\qasc). %
These simulate a 32-node network.
Aggregation runs entirely on CPU, henceforth GPUs affect only local training speed.
All experiments use randomly generated $\delta$-regular graphs, with $\delta \approx \log_2(n)$, a common choice in decentralized networks.
Node count remains fixed throughout.
\Cref{sec:overhead_analysis} confirms that protocol overhead scales with node degree rather than total nodes.
All results are averaged across 5 runs with different seeds. %
Unless stated otherwise, masking requirement $s=1$. \Cref{sec:masking_req_experiments} examines varying $s$.

\textbf{Metrics.}
We evaluate test loss on all datasets and top-1 accuracy on \cifar, \celeba, and \qasc.
Models trained on \xsum are evaluated at training completion using \rougel~\cite{lin2004rouge}, standard for text summarization tasks~\cite{10.1145/3512467}.
We also measure total data transferred per node, categorized in: (i) mask coordination overhead, (ii) masked model parameters, and (iii) metadata (indices).

\textbf{Baselines.}
\dsgd with Metropolis-Hastings weights \cite{XIAO200465} is used as the baseline for \sys.
\sys supports sparsification methods that do not carry information from one round to the next.
Therefore, to test the capabilities of \sys, we run experiments over two prominent sparsification schemes: random subsampling and \topk sparsification.
Both are run for different fractions of selected indices.
We use \topk only on \iid data since previous work has shown that \topk performs poorly in \niid settings~\cite{dhasade:2023:jwins}.

In all experiments, we invert \Cref{th:expected_parameters_shared} using bisection to set sparsification fractions ($\alpha$) such that the fraction of parameters shared with neighbors ($\beta$) in \sys closely matches the \SI{30}{\%} and \SI{50}{\%} marks.
Specifically with masking requirement $s=1$, for nodes with a degree \num{3}, corresponding $\alpha$s are \SI{43.83}{\%} and \SI{59.70}{\%} for \SI{30}{\%} and \SI{50}{\%} parameter sharing, respectively. For nodes of degree \num{6}, these values are \SI{34.22}{\%} and \SI{51.39}{\%}.
For nodes of degree \num{5}, a sparsification fraction of \SI{36.03}{\%} gives \SI{30}{\%} parameter sharing. We do not use the \SI{50}{\%} setting for degree \num{5}.
Furthermore, to ensure that both algorithms share the exact same fraction of parameters, we first run \sys and measure the mean fraction of shared weights across all nodes and runs and use the obtained value as the fraction of parameters shared in \dsgd with the same settings.

Because \sys works on the basis of obfuscating parameters using masks, we cannot compress the masked parameters. To keep the evaluation consistent, the same is done for \dsgd.
On the other hand, Elias gamma compression is used for compressing indices in both algorithms.

\subsection{Performance}

\subsubsection{Evaluation Across Multiple Datasets}
\label{sec:eval_different_datasets}

\begin{figure*}[ht]
	\centering
	\includegraphics{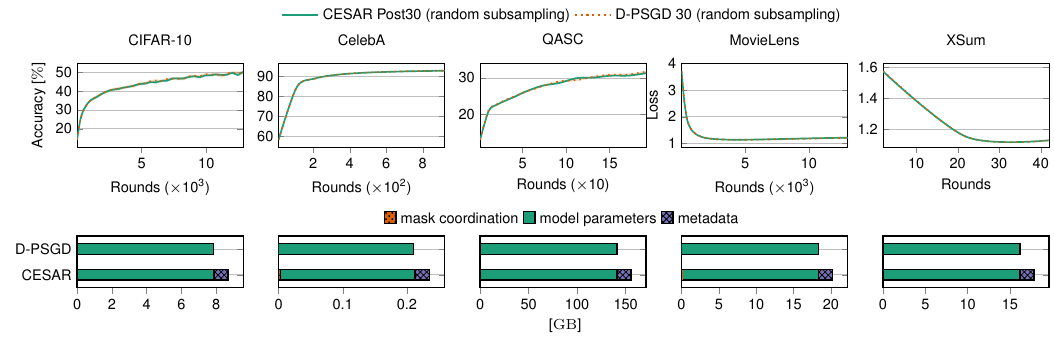}
	\caption{Comparison of performance and per node communication cost of \sys against \dsgd with matching configuration over multiple datasets using random subsampling. \sys is overlapping \dsgd in the performance plots.}
	\label{fig:dataset_eval_comparison_rand}
\end{figure*}

\begin{figure}[ht]
	\centering
	\includegraphics{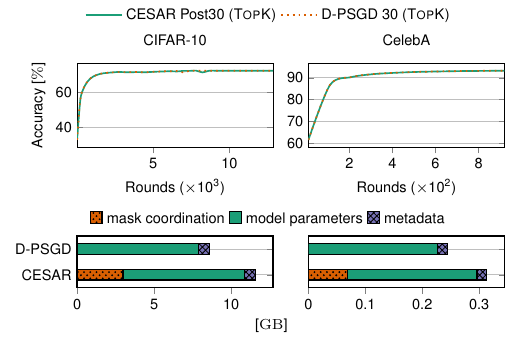}
	\caption{Comparison of performance and per node communication cost of \sys against \dsgd with matching configuration over multiple datasets using \topk. \sys is overlapping \dsgd in the performance plots.}
	\label{fig:dataset_eval_comparison_topk}
\end{figure}

We compare \sys against \dsgd across \cifar, \celeba, \movielens, and \xsum on a 6-regular graph with 30\% parameter sharing.
\qasc uses a 5-regular graph under the same conditions.
\xsum and \qasc run on 32 nodes, the remaining datasets use 48 nodes.

\Cref{fig:dataset_eval_comparison_rand} presents results with random subsampling (\cifar and \movielens \niid; \celeba, \qasc, and \xsum \iid).
\Cref{fig:dataset_eval_comparison_topk} presents results with \topk (\cifar \iid; \celeba \niid).
On \xsum, \tinytransformer achieves \rougel scores of 14.73 with \dsgd and 14.72 with \sys, both starting from 0.00.

\sys matches \dsgd performance while incurring at most 35\% additional data transfer with \topk and at most 11\% with random subsampling.
These results confirm that \sys can replace \dsgd without needing hyperparameter change.

We observe that overhead data is distributed differently between random subsampling and \topk.
In random subsampling, the mask coordination overhead is negligible, as only a seed for generating a random sparsification mask is transmitted.
However, after intersecting these masks, the result must be represented as a full set of indices, requiring more metadata later.
In contrast, \topk always transfers the selection as a full set of indices, which are sent to all second-degree neighbors during the mask coordination, and this scales with the square of the node degree in the network.

\subsubsection{In-Depth Analysis on \cifar}
\label{sec:eval_different_settings}

We analyze \sys on \cifar across 3- and 6-regular graphs at 30\% and 50\% parameter sharing, using \topk for \iid data and random subsampling for \niid.

\Cref{tab:settings_eval} and \Cref{fig:cifar_cesar} present the results.
These empirical findings validate theoretical predictions from \Cref{th:expected_parameters_shared}.

\sys accuracy is comparable to \dsgd across all configurations.
On \niid data with random subsampling, accuracy decreases by up to 0.46\% on 3-regular and 0.18\% on 6-regular graphs.
This reduction stems from precision loss during integer conversion: parameters are converted to integers retaining six decimal places to fit the mask data type.

On \iid data with \topk, \sys outperforms \dsgd by up to 0.3\% on 3-regular and 0.03\% on 6-regular graphs.
We hypothesize this improvement arises because each parameter in \sys is received from at least two neighbors improving model generalization.

Primarily due to mask coordination and the sharing of local sparsification results with neighbors, \sys incurs a data overhead compared to \dsgd with sparsification, ranging between \SI{12}{\%} and \SI{35}{\%} for \topk, and between 7\% and 11\% for random subsampling.
Furthermore, this overhead varies with the degree of the graph.
In random subsampling, it remains nearly constant regardless of the degree since transferring pseudo-random number generator seeds suffices for calculating intersections.
With \topk, however, we observe a significant increase as we move from degree 3 to 6 due to transfer of the selected indices.

\begin{figure}[t]
	\centering
	\includegraphics{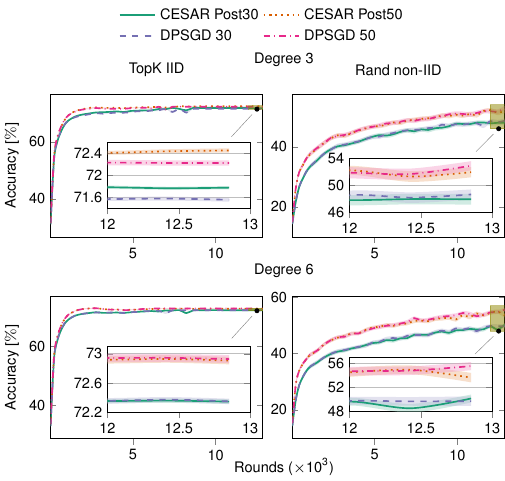}
	\caption{The accuracy comparison of \sys against \dsgd over various configurations, data distributions and network degrees.}
	\label{fig:cifar_cesar}
\end{figure}

\begin{table*}[!ht]
    \centering
    \caption{Performance comparison of \sys against \dsgd over regular graphs of various degrees and fractions of model shared. Time format: HH:MM (hours:minutes). In all experiments, standard error for elapsed time $< 1$ minute.}
    \label{tab:settings_eval}
    \setlength{\tabcolsep}{4pt} %
    \begin{tabular}{cc|c|c|cc|cc|cc}
        \toprule
        \textbf{Distribution} & \textbf{Sparsification} & \textbf{Degree} & \textbf{\% shared} & \multicolumn{2}{c|}{\textbf{Max. Acc. [\%]}} & \multicolumn{2}{c|}{\textbf{Data [GB]}} & \multicolumn{2}{c}{\textbf{Time [HH:MM]}}\\
        & & & \textbf{weights} & \textbf{\sys} & \textbf{D-PSGD} & \textbf{\sys} & \textbf{D-PSGD} & \textbf{\sys} & \textbf{D-PSGD}\\
        \midrule
        \multirow{4}{*}{IID} & \multirow{4}{*}{\topk} & \multirow{2}{*}{3} & 31.02 & 71.88 & 71.58 & 5.22 & 4.41 & 03:19 & 02:49 \\
         & & & 50.48 & 72.50 & 72.29 & 7.90 & 7.03 & 03:30 & 02:53 \\
         \cmidrule{3-10}
         & & \multirow{2}{*}{6} & 30.13 & 72.42 & 72.40 & 11.56 & 8.58 & 03:44 & 02:35\\
         & & & 49.67 & 73.12 & 73.09 & 17.28 & 13.83 & 04:17 & 02:44\\
        \midrule
        \multirow{4}{*}{non-IID} & \multirow{4}{*}{\begin{tabular}{@{}c@{}}random\\subsamp.\end{tabular}} & \multirow{2}{*}{3} & 30.00 & 48.64 & 49.06 & 4.34 & 3.92 & 02:58 & 02:40\\
         & & & 50.00 & 52.54 & 53.00 & 7.01 & 6.53 & 03:04 & 02:41\\
         \cmidrule{3-10}
         & & \multirow{2}{*}{6} & 30.00 & 50.24 & 50.38 & 8.69 & 7.85 & 03:06 & 02:44\\
         & & & 50.00 & 55.44 & 55.62 & 14.04 & 13.07 & 03:21 & 02:45\\
        \bottomrule
    \end{tabular}
\end{table*}

\subsubsection{Comparison with secure aggregation}
\label{sec:full_secagg_comparison}

\begin{figure}[t]
	\centering
	\includegraphics{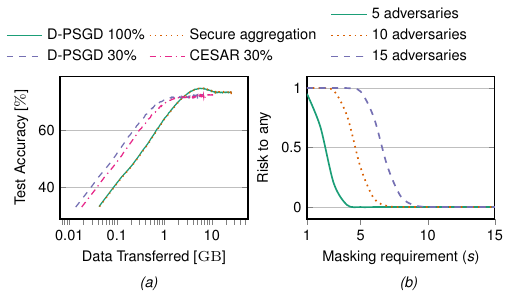}
	\caption{\emph{(a)} Data transferred to reach a certain accuracy for secure aggregation and full sharing against \dsgd and \sys for 30\% sharing. \emph{(b)} Estimated attack risk to any honest node in randomly generated 100 node graph with 25-regular topology and a fixed number of colluding adversaries.}
	\label{fig:3_in_1}
\end{figure}

In this section, we examine how the amount of data transferred in \sys compares directly with secure aggregation.
This analysis is performed on \iid \cifar dataset in a 6-regular 48-node network.
We consider scenarios where 30\% of indices are shared using \sys and \dsgd, as well as cases where all indices are shared with \dsgd and secure aggregation.
The results are shown in \Cref{fig:3_in_1}a.

The results reveal that models trained using \sys not only achieve higher accuracy for the same amount of data compared to secure aggregation, but they also outperform \dsgd with full sharing.
However, this advantage is observed only up to the point where the models trained with \sys converge.
These findings indicate that despite the overheads incurred during \sys, it retains the benefits of integrating sparsification, allowing it to achieve same accuracy at lower cost in comparison to secure aggregation.

\subsubsection{Configuring for Minimized Collusion Risks}
\label{sec:eval_collusion}

\Cref{sec:collusion_theory} discusses formal privacy guarantees in systems with known number of adversaries.
However, this is often not the case, and the number of adversaries has to be estimated.
Therefore, we study how the privacy risk changes as we increase the masking requirement~($s$) in \sys for a  25-regular graph with 100 nodes in settings with different number of adversaries.

This is performed by running a Monte Carlo simulation \num{250000} times for each masking requirement.
Each simulation involves generating a random 25-regular graph and assigning each node to be either adversarial or honest.
All adversarial nodes are assumed to collude with each other.
Over the simulation, we measure the probability that any honest node in the network is at risk of adversaries.
A node is considered to be at risk if any of their parameters can be inferred by the adversaries.
This can only happen when the node is connected to an adversary who has at least $s$ adversarial neighbors.

The simulation results are shown in \Cref{fig:3_in_1}b.
The results reveal that even for masking requirement lower than the total number of adversaries, yet sufficiently high, the risk can be practically nonexistent to any node in the network.
For instance, for the setting with 15 adversaries we do not come across any graphs under risk for masking requirement $\ge 13$ in any of our simulations.
Furthermore, with masking requirement $9$ the risk is $1.45\%$ in this setting.

\subsubsection{Effects of Masking Requirement}
\label{sec:masking_req_experiments}

Having seen how increasing $s$ significantly reduces collusion risk, we now examine its impact on convergence and communication overhead.
All experiments use a 6‐regular 48‐node network sharing \SI{30}{\%} of model parameters, on \cifar with both \iid and \niid splits. We compare two sparsification strategies: \topk and random subsampling.
To maintain a constant \SI{30}{\%} of parameter sharing, raising $s$ requires increasing the percentage of parameters selected in local sparsification~$\alpha$ (from \SI{34.21}{\%} at $s=1$, to \SI{42.53}{\%} at $s=2$, and \SI{53.34}{\%} at $s=3$).

\Cref{fig:masking_requirement_collusion_exp} summarizes convergence and total data exchanged of \sys against \dsgd with 30\% sharing.
With random subsampling the actual percentage of shared parameters stays within \SI{0.001}{\%} of the \SI{30}{\%} target for all $s$, and total data exchanged is within \SI{0.003}{\%} of \dsgd. Maximum accuracy varies by at most \SI{0.4}{\%} points relative to \dsgd, indicating that increasing $s$ has a negligible effect on both learning and communication when using random subsampling, while significantly strengthening the privacy guarantees of \sys. %

Meanwhile, with \topk, as $s$ grows, maintaining a \SI{30}{\%} share becomes more challenging: at $s=3$, we end up sharing \SI{32.13}{\%}. This higher share naturally results in up to \SI{0.3}{\%} points better accuracy than \dsgd. However, mask‐coordination overhead also rises, since each node shares a larger set of indices selected in local sparsification. In fact, moving from $s=1$ to $s=3$ increases total data exchanged by \SI{9.1}{\%}. Thus, as $s$ is increased, alongside improving privacy guarantees, \sys incurs a marginal communication overhead under \topk sparsification.

\begin{figure}[t]
	\centering
	\includegraphics{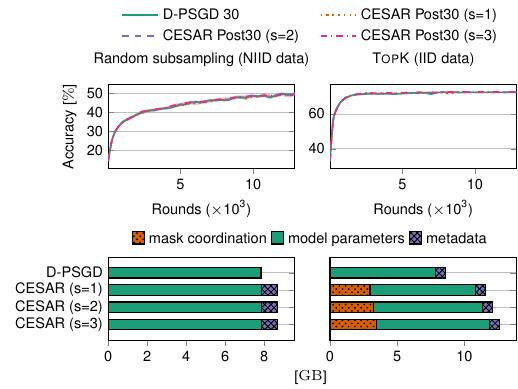}
	\caption{Comparison of the convergence of \sys without collusion protection (s=1) and with protection against two (s=2) and three (s=3) colluding nodes, against \dsgd for the same inteded percentage of shared parameters.}
	\label{fig:masking_requirement_collusion_exp}
\end{figure}

\subsubsection{Privacy leakage from partial sums}
\label{sec:exp:partial_leakage_eval}

We have seen how \sys{} keeps local models private in the \ac{HbC} setting and how these guarantees extend to the presence of colluding nodes. However, even if the exact parameters of a local model remain private, adversaries can collude to learn aggregated information. For example, consider node $N_r$ with local model $\theta_r$ and neighbors $N_1$, $N_2$, and $N_3$, with models $\theta_1$, $\theta_2$, and $\theta_3$, respectively. Suppose $N_r$ colludes with $N_1$ to maximize information extraction. While $\theta_2$ and $\theta_3$ individually remain hidden, the adversaries can infer the sum $\theta_2 + \theta_3$ by subtracting their known models $\theta_r$ and $\theta_1$ from the collective sum $\theta_r + \theta_1 + \theta_2 + \theta_3$.

In this section, we analyze privacy leakage as a function of the number of models in a partial sum.
We evaluate on a 32-node 5-regular graph using \niid{} \cifar{}. This partitioning ensures that local training distributions are both heterogeneous across nodes and distant from the test distribution, thereby amplifying any leakage and enabling clear comparisons. At the end of each global epoch, every node selects a random neighbor and launches a threshold \ac{MIA}~\cite{yeom2018privacyriskmachinelearning} on the last model received from them, measuring the \ac{ROC-AUC}. It then adds a model of another neighbor (from the same communication round) to the model of the victim, averages the two, and reruns the attack. This process is repeated, cumulatively adding one neighbor at a time, until models of all neighbors have been incorporated.

We evaluate leakage only for the initially selected victim: \emph{member} dataset is the entire local training set of the victim, and \emph{non-member} dataset is the global test set. Although in realistic scenarios an adversary would not access the full training set, this setup provides an upper bound on privacy leakage.

We train using \dsgd{} and repeat the experiment under three sharing strategies: full parameter sharing, random subsampling, and \topk{} sparsification with \SI{30}{\%} parameter selection. For each fixed partial-sum size, we average results over all nodes and communication rounds. \Cref{fig:mia_weakness} shows the privacy leakage (\ac{ROC-AUC}) as a function of the number of models aggregated in the sum.

These results offer several insights:
\begin{inparaenum}[(i)]
  \item Applying sparsification already reduces leakage compared to full sharing, although different sparsification methods offer varying privacy benefits.
  \item Partial sums leak more than fully aggregated models but less than individual local models, indicating that secure aggregation still mitigates leakage under collusion.
  \item There is a sharp privacy gain as soon as a partial sum includes more than the model of the victim alone.
  \item Compared to a single local model, the fully aggregated model leaks substantially less: its \ac{ROC-AUC} decreases by 0.24 under full sharing, 0.17 under \topk{}, and 0.11 under random subsampling.
\end{inparaenum}

\begin{figure}[ht]
	\centering
    \includegraphics{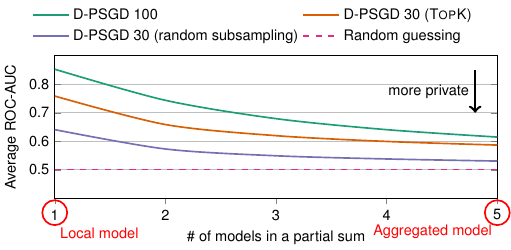}
	\caption{Privacy leakage from partial sums under a threshold attack on \niid{} \cifar.}
	\label{fig:mia_weakness}
\end{figure}

\subsubsection{Effect of dropouts}
\label{sec:exp:dropout_handling}

\begin{figure}[ht]
	\centering
    \includegraphics{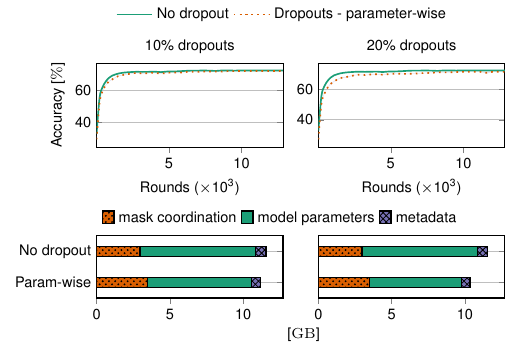}
	\caption{Comparison of model accuracy and data exchanged by \sys without dropouts, and with the parameter-wise dropout handling add-on under \SI{10}{\%} and \SI{20}{\%} dropouts.} %
	\label{fig:dropout_handling}
\end{figure}

\Cref{sec:handling_dropouts} modifies \sys to handle node dropouts. We evaluate this on a \num{48}-node \num{6}-regular network training on \iid \cifar with \topk sparsification. Each node selects \num{34.22}\% of parameters, yielding approximately \num{30}\% sharing after discarding unmasked parameters; the masking requirement is $s=1$.

Each round, nodes crash independently with fixed probability and rejoin the next round. Crashes are simulated, with neighbors notified directly rather than by timeout. We test \num{10}\% and \num{20}\% crash rates under worst-case timing: dropouts occur after mask coordination but before masked models are sent.

The results of this experiment are shown in \Cref{fig:dropout_handling}.
In this setting, \sys achieves a maximum accuracy of \num{72.42}\% without dropouts and reaches \num{72.06}\% and \num{71.51}\% accuracy for \num{10}\% and \num{20}\% dropouts, respectively, when the parameter-wise dropout handling is used.
These results demonstrate even with dropouts, the maximum achieved accuracy and convergence rate remain relatively stable, although higher dropout rates result in lower final accuracy.

Additionally, without dropouts, \sys exchanges \SI{11.58}{GB} of data. With parameter-wise dropout handling, the system exchanges \SI{11.19}{GB} and \SI{10.34}{GB} of data for \num{10}\% and \num{20}\% dropouts, respectively. This shows that while the add-on requires additional communication (due to sharing indices selected during local sparsification with immediate neighbors), the overall data exchanged is still reduced, as crashed nodes do not share their masked models.

Finally, without the add-on and without dropouts, the experiment takes an average of 4 hours and 23 minutes to complete with \sys. In contrast, using the add-on increases the runtime to 5 hours and 4 minutes for \num{10}\% dropouts and 5 hours and 18 minutes for \num{20}\% dropouts, representing an increase in execution time of \num{16}\% and \num{21}\%, respectively.

\subsection{Summary of Findings}
Our systematic evaluation confirms that \sys successfully reconciles the trade-offs inherent to \ac{DL}.
Regarding \textbf{RQ1}, we demonstrated that \sys matches \dsgd accuracy across all tasks while reducing data exchange by up to \SI{66}{\percent} compared to full-parameter schemes.
For \textbf{RQ2} and \textbf{RQ3}, our results show that increasing the masking requirement $s$ effectively neutralizes collusion risks with negligible impact on convergence, while providing stronger defense against \ac{MIA} than standard aggregation.
Finally, addressing \textbf{RQ4}, we showed that our parameter-wise dropout handling maintains high accuracy even under \SI{20}{\percent} failure rates.
Collectively, these results validate \sys as a robust, general-purpose protocol for private and efficient \ac{DL}.

\section{Related Work}
\label{sec:rw}

\textbf{Privacy attacks.}
Gradient inversion~\cite{zhu2019deep}, \ac{MIA}~\cite{shokriMembershipInferenceAttacks2017}, and attribute inference~\cite{zhao2021feasibility} attacks extract training-set information from exchanged models, including in \dsgd~\cite{pasquini2023security}.
Although sparsification has been proposed as a defense~\cite{shokri2015privacy}, Yue~\etal~\cite{yue2023gradient} demonstrate gradient inversion on \topk-sparsified gradients.
These attacks remain poorly understood under secure aggregation.
Because attacks on averaged models affect \sys and conventional secure aggregation equally, we consider them orthogonal to this work.

\textbf{Secure aggregation.}
Secure aggregation conceals individual updates while revealing only their aggregate, and has become a standard privacy mechanism in \ac{FL}.
Most protocols build on Bonawitz~\etal~\cite{bonawitz2017practical}, using pairwise additive masks and Shamir secret sharing~\cite{shamir1979share} for dropout recovery, typically coordinated by an untrusted server.
Subsequent work reduces communication: BBGLR~\cite{cryptoeprint:2020/704} limits each client's mask partners to a small subset, and Flamingo~\cite{cryptoeprint:2023/486} amortizes mask agreement across rounds.
These optimizations target \ac{FL}'s single large-scale aggregation per round and they do not transfer directly to \ac{DL}, which performs many small aggregations among few peers.
In comparison prior secure-aggregation work in \ac{DL} is limited~\cite{GUPTA20179515}.
\sys is the first protocol to integrate sparsification without additional servers, preserving the fully decentralized model.

\textbf{Secure aggregation with sparsification.}
Model-size reduction via compression~\cite{collet:lz4:2022}, sparsification~\cite{alistarh2018sparseconvergence}, and quantization~\cite{seide2014-bit} is well studied, yet its integration with secure aggregation remains limited to \ac{FL}.
SparseSecAgg~\cite{ergun2021sparsified} embeds a fixed sparsification mechanism into the aggregation protocol, restricting compatibility with other methods.
Lu~\etal~\cite{lu2023top} incorporate \topk by masking the \emph{union} of selected indices across clients.
Both approaches rely on a central server and therefore cannot be trivially adapted to \ac{DL}.
In contrast, \sys masks only the \emph{intersection} of locally selected indices and operates without any server. 
\section{Conclusion}
\label{sec:conclusion}

\sys is the first secure aggregation protocol for \ac{DL} that natively supports sparsification and node dropouts without requiring any server.
\sys reconciles the tension between privacy, communication efficiency, and utility by masking only the intersection of locally selected indices, ensuring that pairwise masks cancel upon aggregation.
Formal analysis establishes privacy guarantees against \ac{HbC} and colluding adversaries.
Experiments on models with up to 124M parameters show that \sys matches the accuracy of state-of-the-art baselines while reducing data exchange by 66\% compared to full-parameter sharing, and incurring only a 7–35\% communication overhead compared to plain (unsecure) sparsification.
Further optimization of \sys for various network topologies and more complex models is a promising direction for future work.
Integrating \sys into real-world applications holds great potential to advance secure, efficient, and privacy-preserving \ac{DL} systems significantly. 

\bibliographystyle{IEEEtran}
\bibliography{IEEEabrv,references/main}

\begin{thebibliography}{10}
\providecommand{\url}[1]{#1}
\csname url@samestyle\endcsname
\providecommand{\newblock}{\relax}
\providecommand{\bibinfo}[2]{#2}
\providecommand{\BIBentrySTDinterwordspacing}{\spaceskip=0pt\relax}
\providecommand{\BIBentryALTinterwordstretchfactor}{4}
\providecommand{\BIBentryALTinterwordspacing}{\spaceskip=\fontdimen2\font plus
\BIBentryALTinterwordstretchfactor\fontdimen3\font minus
  \fontdimen4\font\relax}
\providecommand{\BIBforeignlanguage}[2]{{%
\expandafter\ifx\csname l@#1\endcsname\relax
\typeout{** WARNING: IEEEtran.bst: No hyphenation pattern has been}%
\typeout{** loaded for the language `#1'. Using the pattern for}%
\typeout{** the default language instead.}%
\else
\language=\csname l@#1\endcsname
\fi
#2}}
\providecommand{\BIBdecl}{\relax}
\BIBdecl

\bibitem{ramezan_effects_2021}
C.~A. Ramezan, T.~A. Warner, A.~E. Maxwell, and B.~S. Price, ``Effects of
  training set size on supervised machine-learning land-cover classification of
  large-area high-resolution remotely sensed data,'' \emph{Remote Sensing},
  vol.~13, no.~3, p. 368, 2021.

\bibitem{sheller2018multiinstitutionaldeeplearningmodeling}
M.~J. Sheller, G.~A. Reina, B.~Edwards, J.~Martin, and S.~Bakas,
  ``Multi-institutional deep learning modeling without sharing patient data: A
  feasibility study on brain tumor segmentation,'' 2018.

\bibitem{firstdlpaper}
A.~Lalitha, S.~Shekhar, T.~Javidi, and F.~Koushanfar, ``Fully decentralized
  federated learning,'' in \emph{Third workshop on bayesian deep learning
  (NeurIPS)}, vol.~2, 2018.

\bibitem{keyboardImprovementFL}
T.~Yang, G.~Andrew, H.~Eichner, H.~Sun, W.~Li, N.~Kong, D.~Ramage, and
  F.~Beaufays, ``Applied federated learning: Improving google keyboard query
  suggestions,'' 2018.

\bibitem{radiationoncologyFL}
D.~Jarrett, E.~Stride, K.~Vallis, and M.~J. Gooding, ``Applications and
  limitations of machine learning in radiation oncology,'' \emph{The British
  journal of radiology}, vol.~92, no. 1100, p. 20190001, 2019.

\bibitem{lian2017dpsgd}
X.~Lian, C.~Zhang, H.~Zhang, C.-J. Hsieh, W.~Zhang, and J.~Liu, ``Can
  decentralized algorithms outperform centralized algorithms? a case study for
  decentralized parallel stochastic gradient descent,'' in \emph{NIPS}, 2017.

\bibitem{alistarh2018sparseconvergence}
D.~Alistarh, T.~Hoefler, M.~Johansson, S.~Khirirat, N.~Konstantinov, and
  C.~Renggli, ``The convergence of sparsified gradient methods,'' in
  \emph{NeurIPS}, 2018.

\bibitem{tang2020sparsification}
Z.~Tang, S.~Shi, and X.~Chu, ``Communication-efficient decentralized learning
  with sparsification and adaptive peer selection,'' in \emph{2020 IEEE 40th
  International Conference on Distributed Computing Systems (ICDCS)}, 2020.

\bibitem{lin2020deepgradientcompressionreducing}
Y.~Lin, S.~Han, H.~Mao, Y.~Wang, and W.~J. Dally, ``Deep gradient compression:
  Reducing the communication bandwidth for distributed training,'' 2020.

\bibitem{shokriMembershipInferenceAttacks2017}
R.~Shokri, M.~Stronati, C.~Song, and V.~Shmatikov, ``Membership inference
  attacks against machine learning models,'' in \emph{2017 {{IEEE Symposium}}
  on {{Security}} and {{Privacy}}, {{SP}} 2017, {{San Jose}}, {{CA}}, {{USA}},
  {{May}} 22-26, 2017}, 2017, pp. 3--18.

\bibitem{carlini2022membership}
N.~Carlini, S.~Chien, M.~Nasr, S.~Song, A.~Terzis, and F.~Tramer, ``Membership
  inference attacks from first principles,'' in \emph{2022 IEEE Symposium on
  Security and Privacy (SP)}.\hskip 1em plus 0.5em minus 0.4em\relax IEEE,
  2022, pp. 1897--1914.

\bibitem{cyffers2022muffliato}
E.~Cyffers, M.~Even, A.~Bellet, and L.~Massouli\'{e}, ``Muffliato: Peer-to-peer
  privacy amplification for decentralized optimization and averaging,'' in
  \emph{Advances in Neural Information Processing Systems}, S.~Koyejo,
  S.~Mohamed, A.~Agarwal, D.~Belgrave, K.~Cho, and A.~Oh, Eds., vol.~35.\hskip
  1em plus 0.5em minus 0.4em\relax Curran Associates, Inc., 2022, pp.
  15\,889--15\,902.

\bibitem{biswas2025low}
S.~Biswas, D.~Frey, R.~Gaudel, A.-M. Kermarrec, D.~Ler{\'e}v{\'e}rend,
  R.~Pires, R.~Sharma, and F.~Ta{\"\i}ani, ``Low-cost privacy-preserving
  decentralized learning,'' \emph{Proceedings on Privacy Enhancing
  Technologies}, 2025.

\bibitem{biswas2025noiseless}
S.~Biswas, M.~Even, A.-M. Kermarrec, L.~Massoulié, R.~Pires, R.~Sharma, and
  M.~de~Vos, ``Noiseless privacy-preserving decentralized learning,''
  \emph{Proceedings on Privacy Enhancing Technologies}, vol. 2025, no.~1, p.
  824–844, Jan. 2025.

\bibitem{bonawitz2017practical}
K.~Bonawitz, V.~Ivanov, B.~Kreuter, A.~Marcedone, H.~B. McMahan, S.~Patel,
  D.~Ramage, A.~Segal, and K.~Seth, ``Practical secure aggregation for
  privacy-preserving machine learning,'' in \emph{proceedings of the 2017 ACM
  SIGSAC Conference on Computer and Communications Security}, 2017, pp.
  1175--1191.

\bibitem{ergun2021sparsified}
I.~Ergun, H.~U. Sami, and B.~Guler, ``Sparsified secure aggregation for
  privacy-preserving federated learning,'' \emph{arXiv preprint
  arXiv:2112.12872}, 2021.

\bibitem{lu2023top}
S.~Lu, R.~Li, W.~Liu, C.~Guan, and X.~Yang, ``Top-k sparsification with secure
  aggregation for privacy-preserving federated learning,'' \emph{Computers \&
  Security}, vol. 124, p. 102993, 2023.

\bibitem{acs2011dream}
G.~{\'A}cs and C.~Castelluccia, ``I have a dream!(differentially private smart
  metering),'' in \emph{International Workshop on Information Hiding}.\hskip
  1em plus 0.5em minus 0.4em\relax Springer, 2011, pp. 118--132.

\bibitem{jelasity2007gossip}
M.~Jelasity, S.~Voulgaris, R.~Guerraoui, A.-M. Kermarrec, and M.~Van~Steen,
  ``Gossip-based peer sampling,'' \emph{ACM Transactions on Computer Systems
  (TOCS)}, vol.~25, no.~3, pp. 8--es, 2007.

\bibitem{guerraoui2024peerswap}
R.~Guerraoui, A.-M. Kermarrec, A.~Kucherenko, R.~Pinot, and M.~de~Vos,
  ``Peerswap: A peer-sampler with randomness guarantees,'' in \emph{Proceedings
  of the 43rd International Symposium on Reliable Distributed Systems (SRDS
  2024)}, 2024.

\bibitem{chandra1996unreliable}
T.~D. Chandra and S.~Toueg, ``Unreliable failure detectors for reliable
  distributed systems,'' \emph{Journal of the ACM (JACM)}, vol.~43, no.~2, pp.
  225--267, 1996.

\bibitem{decentralizepy}
A.~Dhasade, A.-M. Kermarrec, R.~Pires, R.~Sharma, and M.~Vujasinovic,
  ``Decentralized learning made easy with decentralizepy,'' in
  \emph{Proceedings of the 3rd Workshop on Machine Learning and Systems}, 2023,
  pp. 34--41.

\bibitem{bach2022promptsource}
S.~H. Bach, V.~Sanh, Z.-X. Yong, A.~Webson, C.~Raffel, N.~V. Nayak, A.~Sharma,
  T.~Kim, M.~S. Bari, T.~Fevry \emph{et~al.}, ``Promptsource: An integrated
  development environment and repository for natural language prompts,''
  \emph{arXiv preprint arXiv:2202.01279}, 2022.

\bibitem{hsiehskewscout2020}
K.~Hsieh, A.~Phanishayee, O.~Mutlu, and P.~B. Gibbons, ``The non-{IID} data
  quagmire of decentralized machine learning,'' in \emph{ICML}, 2020.

\bibitem{gpt2paper}
A.~Radford, J.~Wu, R.~Child, D.~Luan, D.~Amodei, I.~Sutskever \emph{et~al.},
  ``Language models are unsupervised multitask learners,'' \emph{OpenAI blog},
  vol.~1, no.~8, p.~9, 2019.

\bibitem{korenmatrixfactorization2009}
Y.~Koren, R.~Bell, and C.~Volinsky, ``Matrix factorization techniques for
  recommender systems,'' \emph{Computer}, vol.~42, no.~8, pp. 30--37, aug 2009.

\bibitem{2020t5}
C.~Raffel, N.~Shazeer, A.~Roberts, K.~Lee, S.~Narang, M.~Matena, Y.~Zhou,
  W.~Li, and P.~J. Liu, ``Exploring the limits of transfer learning with a
  unified text-to-text transformer,'' \emph{Journal of Machine Learning
  Research}, vol.~21, no. 140, pp. 1--67, 2020.

\bibitem{lin2004rouge}
C.-Y. Lin, ``Rouge: A package for automatic evaluation of summaries,'' in
  \emph{Text summarization branches out}, 2004, pp. 74--81.

\bibitem{10.1145/3512467}
W.~Yu, C.~Zhu, Z.~Li, Z.~Hu, Q.~Wang, H.~Ji, and M.~Jiang, ``A survey of
  knowledge-enhanced text generation,'' \emph{ACM Comput. Surv.}, vol.~54, no.
  11s, nov 2022.

\bibitem{XIAO200465}
L.~Xiao and S.~Boyd, ``Fast linear iterations for distributed averaging,''
  \emph{Systems \& Control Letters}, vol.~53, no.~1, pp. 65--78, 2004.

\bibitem{dhasade:2023:jwins}
A.~Dhasade, A.-M. Kermarrec, R.~{Pires}, R.~Sharma, M.~Vujasinovic, and
  J.~Wigger, ``Get more for less in decentralized learning systems,'' in
  \emph{2023 IEEE 43rd International Conference on Distributed Computing
  Systems (ICDCS '23)}, 2023, pp. 463--474.

\bibitem{yeom2018privacyriskmachinelearning}
S.~Yeom, I.~Giacomelli, M.~Fredrikson, and S.~Jha, ``Privacy risk in machine
  learning: Analyzing the connection to overfitting,'' 2018.

\bibitem{zhu2019deep}
L.~Zhu, Z.~Liu, and S.~Han, ``Deep leakage from gradients,'' \emph{Advances in
  neural information processing systems (NeurIPS '19)}, vol.~32, 2019.

\bibitem{zhao2021feasibility}
B.~Z.~H. Zhao, A.~Agrawal, C.~Coburn, H.~J. Asghar, R.~Bhaskar, M.~A. Kaafar,
  D.~Webb, and P.~Dickinson, ``On the (in) feasibility of attribute inference
  attacks on machine learning models,'' in \emph{2021 IEEE European Symposium
  on Security and Privacy (EuroS\&P)}.\hskip 1em plus 0.5em minus 0.4em\relax
  IEEE, 2021, pp. 232--251.

\bibitem{pasquini2023security}
D.~Pasquini, M.~Raynal, and C.~Troncoso, ``On the (in) security of peer-to-peer
  decentralized machine learning,'' in \emph{2023 IEEE Symposium on Security
  and Privacy (SP)}, 2023, pp. 418--436.

\bibitem{shokri2015privacy}
R.~Shokri and V.~Shmatikov, ``Privacy-preserving deep learning,'' in
  \emph{Proceedings of the 22nd ACM SIGSAC conference on computer and
  communications security}, 2015, pp. 1310--1321.

\bibitem{yue2023gradient}
K.~Yue, R.~Jin, C.-W. Wong, D.~Baron, and H.~Dai, ``Gradient obfuscation gives
  a false sense of security in federated learning,'' in \emph{32nd USENIX
  Security Symposium (USENIX Security 23)}, 2023, pp. 6381--6398.

\bibitem{shamir1979share}
A.~Shamir, ``How to share a secret,'' \emph{Communications of the ACM},
  vol.~22, no.~11, pp. 612--613, 1979.

\bibitem{cryptoeprint:2020/704}
J.~Bell, K.~A. Bonawitz, A.~Gascón, T.~Lepoint, and M.~Raykova, ``Secure
  single-server aggregation with (poly)logarithmic overhead,'' Cryptology
  {ePrint} Archive, Paper 2020/704, 2020.

\bibitem{cryptoeprint:2023/486}
Y.~Ma, J.~Woods, S.~Angel, A.~Polychroniadou, and T.~Rabin, ``Flamingo:
  Multi-round single-server secure aggregation with applications to private
  federated learning,'' Cryptology {ePrint} Archive, Paper 2023/486, 2023.

\bibitem{GUPTA20179515}
N.~Gupta, J.~Katz, and N.~Chopra, ``Privacy in distributed average consensus,''
  \emph{IFAC-PapersOnLine}, vol.~50, no.~1, pp. 9515--9520, 2017, 20th IFAC
  World Congress.

\bibitem{collet:lz4:2022}
Y.~Collet, ``{LZ4},'' 2022.

\bibitem{seide2014-bit}
F.~Seide, H.~Fu, J.~Droppo, G.~Li, and D.~Yu, ``1-bit stochastic gradient
  descent and application to data-parallel distributed training of speech
  {DNNs},'' in \emph{Interspeech 2014}, September 2014.

\end{thebibliography}

\appendices

\section{Scaling with Network Size}

In this section, we examine how the communication cost of \sys scales with an increasing number of nodes in the network. 
Measurements were conducted for networks comprising 48, 96, 192, and 288 nodes. 
Each experiment was performed within a 5-regular network using the \cifar dataset. 
The dataset was divided into 288 equally-sized segments, corresponding to the maximum node count, with each node receiving one segment. 
In cases where the number of segments surpassed the number of nodes, the excess segments remained unassigned and were not utilized in the training process.

\begin{figure}[ht]
    \centering
    \includegraphics{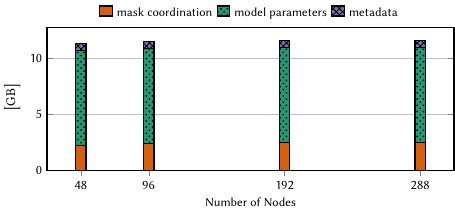}
    \caption{Amount of data transferred per node during training for different number of nodes}
    \label{fig:network_size_data_dependance}
\end{figure}

\begin{table}[ht]
    \centering
    \caption{Amount of data transferred per node during training using \sys for a fixed setup in networks of different sizes}
    \label{tab:network_size_data_dependance}
\pgfplotstableset{
    create on use/param_data/.style={
    create col/expr={\thisrow{total_bytes_mean} + \thisrow{total_meta_mean}}}
}

\begin{filecontents*}{data/CIFAR/scalability/cummulative_avg_data.csv}
experiment_name,total_bytes_mean,total_protocol_overhead_mean,total_meta_mean,total_data_per_n_mean,absolute_total_mean,total_bytes_std,total_protocol_overhead_std,total_meta_std,total_data_per_n_std,absolute_total_std,total_bytes_nr_nodes,total_protocol_overhead_nr_nodes,total_meta_nr_nodes,total_data_per_n_nr_nodes,absolute_total_nr_nodes,total_bytes_error,total_protocol_overhead_error,total_meta_error,total_data_per_n_error,absolute_total_error,number_of_nodes
scalability-CESAR-TopKpost40-5regular-IID-48nodes,8.442071623230975,2.246563016087748,0.625825192744378,8.442071623230975,11.3144598320631,0.0043495761291764,0.1669483883790378,0.0019199784196567,0.0043495761291764,0.1666772736956361,240.0,240.0,240.0,240.0,240.0,0.0005502973064378,0.0211218853770272,0.0002429107851839,0.0005502973064378,0.0210875846370055,48
scalability-CESAR-TopKpost40-5regular-IID-96nodes,8.44504215764658,2.3874338156446657,0.6224172712652944,8.44504215764658,11.45489324455654,0.0044212070640007,0.1454136844153162,0.0016029290610477,0.0044212070640007,0.1455125832641906,480.0,480.0,480.0,480.0,480.0,0.0003955271572578,0.013008904670287,0.0001434001994533,0.0003955271572578,0.0130177522949255,96
scalability-CESAR-TopKpost40-5regular-IID-192nodes,8.440943441772834,2.4789718974459296,0.6197024427781191,8.440943441772834,11.539617781996885,0.0049862445756181,0.0933545644562367,0.0023180009873946,0.0049862445756181,0.0933411065196795,960.0,960.0,960.0,960.0,960.0,0.0003154234892909,0.0059054909993803,0.0001466337939376,0.0003154234892909,0.0059046396674324,192
scalability-CESAR-TopKpost40-5regular-IID-288nodes,8.437486876306744,2.493632216180494,0.6171665341341092,8.437486876306744,11.548285626621348,0.0052816136309754,0.080335041425966,0.0023241833382927,0.0052816136309754,0.0805777359145417,1440.0,1440.0,1440.0,1440.0,1440.0,0.0002727981703162,0.0041493478782238,0.0001200453131307,0.0002727981703162,0.0041618831784285,288
\end{filecontents*}
\pgfplotstabletypeset[
    col sep=comma,
    columns={number_of_nodes, param_data, total_protocol_overhead_mean}, %
    columns/number_of_nodes/.style={column name=\textbf{Number of nodes}, column type={{C{20mm}}}},
    columns/param_data/.style={
        column name=\begin{tabular}{@{}c@{}}\textbf{Parameter data}\\\textbf{[\si{GB}]}\end{tabular}, column type={|c},
        precision=2},
    columns/total_protocol_overhead_mean/.style={
        column name=\begin{tabular}{@{}c@{}}\textbf{CESAR overhead}\\\textbf{[\si{GB}]}\end{tabular}, column type={|c},
        precision=2},
    every head row/.style={before row=\toprule, after row=\midrule},
    every last row/.style={after row=\bottomrule},
]{data/CIFAR/scalability/cummulative_avg_data.csv} \end{table}

\Cref{fig:network_size_data_dependance} illustrates that the communication overhead of \sys remains practically independent of the number of nodes in the network. 
However, a more detailed examination in \Cref{tab:network_size_data_dependance} reveals a marginally higher \sys overhead for larger networks. 
This phenomenon can be attributed to the experimental setup. 
As discussed in \Cref{sec:overhead_analysis}, the number of messages sent during the prestep phase of \sys is proportional to the number of second degree neighbors. 
With fewer nodes in the network, while maintaining a fixed topology, it is more likely for two neighbors to share a common neighbor. 
This overlap effectively reduces the average second degree neighborhood size, thereby lowering the overhead in the prestep.

\section{Scaling with Node Degree}
\label[appendix]{sec:eval_scalability}

In this section, we examine how the communication overhead of \sys scales in practice. 
As observed in \Cref{sec:eval_different_datasets} and \Cref{sec:eval_different_settings}, the overhead with random subsampling remains constant regardless of the network topology. 
In contrast, it increases when \topk is applied.

\Cref{sec:overhead_analysis} derives that the communication overhead of the mask coordination in \sys is $(\alpha d \delta_{\text{max}}^2)$, as a node sends their indices to each second-degree neighbor, and their number typically scales with the square of the degree in the network. 
To verify this, we ran \sys with \topk on a 96-node $\delta$-regular network with $\delta$ values of 3, 6, 9, and 12. 
In each run, 40\% of parameters were selected for sparsification, and we recorded only the overhead caused by the mask coordination phase.

The results for the given degrees are shown in \Cref{fig:degree_scalability_prestep_overhead}. 
They indicate that the overhead scales linearly, contrary to the quadratic scaling we expected. 
However, \Cref{tab:degree_scalability_prestep_overhead} also reveals that the overhead of the mask coordination actually scales with the size of the second-degree neighborhood. 
Because the network size is limited, it is common for two nodes to share a neighbor, hence reducing the size of the second-degree neighborhood compared to a theoretical model where the network size is infinite and nodes sharing a neighbor are uncommon. 
This also implies that the total overhead of the mask coordination is limited by the size of the network to $(\alpha d |\mathcal{N}|)$. 
This overlap of neighbors highlights the importance of the ability of \sys to reuse the same mask that two nodes agreed upon over many aggregation with multiple neighbors these two nodes may have in common.

\begin{table}[ht]
   \centering
   \caption{Amount of data transferred per node in mask coordination of \sys during training for different graph degrees and their relation to the size of second-degree neighborhood.}
   \label{tab:degree_scalability_prestep_overhead}
\pgfplotstableset{
   create on use/overhead_per_2_neighbours/.style={
   create col/expr={\thisrow{total_protocol_overhead_mean} / \thisrow{second_degree_neighbors}}}
}

\begin{filecontents*}{data/CIFAR/degree_scalability/cummulative_avg_data.csv}
experiment_name,total_bytes_mean,total_protocol_overhead_mean,total_meta_mean,total_data_per_n_mean,absolute_total_mean,total_bytes_std,total_protocol_overhead_std,total_meta_std,total_data_per_n_std,absolute_total_std,total_bytes_nr_nodes,total_protocol_overhead_nr_nodes,total_meta_nr_nodes,total_data_per_n_nr_nodes,absolute_total_nr_nodes,total_bytes_error,total_protocol_overhead_error,total_meta_error,total_data_per_n_error,absolute_total_error,degree,number_of_nodes,second_degree_neighbors
degree-scalability-CESAR-TopKpost40-3regular-IID-96nodes,2.109107996678601,0.6306547598544664,0.2411896521943466,2.109107996678601,2.980952408727413,0.0050853307012779,0.0279663978209408,0.0009957699438943,0.0050853307012779,0.0283577659205823,480.0,480.0,480.0,480.0,480.0,0.000454940555119,0.0025019117333197,8.908292452490786e-05,0.000454940555119,0.0025369240522751,3,96,5.925
degree-scalability-CESAR-TopKpost40-6regular-IID-96nodes,6.4117426137129465,2.770483577765602,0.5712106873339508,6.4117426137129465,9.7534368788125,0.0087553439568178,0.1486659400095202,0.0038142410785387,0.0087553439568178,0.1485364939600529,480.0,480.0,480.0,480.0,480.0,0.0007832648993647,0.0132998558497342,0.0003412271601514,0.0007832648993647,0.0132882754312596,6,96,26.866666666666667
degree-scalability-CESAR-TopKpost40-9regular-IID-96nodes,10.60581336626783,5.492486412929915,0.8928631237436396,10.60581336626783,16.991162902941383,0.007076819620112,0.2940700566821563,0.005802890302942,0.007076819620112,0.2962383887055181,480.0,480.0,480.0,480.0,480.0,0.0006331018444173,0.0263079045768345,0.0005191344065493,0.0006331018444173,0.0265018864891893,9,96,53.479166666666664
degree-scalability-CESAR-TopKpost40-12regular-IID-96nodes,14.6642827541645,7.76988177348588,1.2181482803057102,14.6642827541645,23.65231280795609,0.0093847164065748,0.2959477160817457,0.0081070256193166,0.0093847164065748,0.2982311034290704,480.0,480.0,480.0,480.0,480.0,0.0008395694090394,0.0264758825235507,0.0007252654649754,0.0008395694090394,0.0266801574406338,12,96,75.1625
\end{filecontents*}
\pgfplotstabletypeset[
   col sep=comma,
   columns={degree,total_protocol_overhead_mean,second_degree_neighbors,overhead_per_2_neighbours}, %
   columns/degree/.style={column name=\textbf{Degree}, column type={c}},
   columns/second_degree_neighbors/.style={
   		column name={\textbf{Avg. $2^{nd}$ deg. neighborhood size}}, column type={|C{20mm}},
   		precision=2},
   columns/total_protocol_overhead_mean/.style={
       column name={\textbf{Overhead [\si{GB}]}}, column type={|C{14mm}},
       fixed, fixed zerofill, precision=2},
   columns/overhead_per_2_neighbours/.style={
       column name={\textbf{Overhead per $2^{nd}$ deg. neighbor [\si{GB}]}}, column type={|C{23mm}},
       fixed, fixed zerofill, precision=4},
   every head row/.style={before row=\toprule, after row=\midrule},
   every last row/.style={after row=\bottomrule},
]{data/CIFAR/degree_scalability/cummulative_avg_data.csv}
 \end{table}

\begin{figure}[t]
    \centering
    \includegraphics{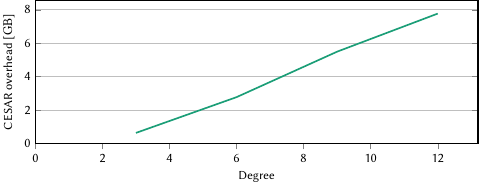}
    \caption{Amount of data transferred per node in mask coordination of \sys during training for different graph degrees.}
    \label{fig:degree_scalability_prestep_overhead}
\end{figure}

\ifthenelse{\boolean{isextendedversion}}{%
\section{Table Of Notation}
\label[appendix]{app:sec:table_of_notation}

In this section, we present a table of symbols used throughout the paper, along with their corresponding descriptions. The table is shown in \Cref{tab:table_of_notation}.

\begin{table*}[h!]
    \centering
    \begin{tabular}{c|p{13cm}}
        \toprule
        \textbf{Notation} & \textbf{Description}\\
        \midrule
        \multicolumn{2}{c}{\textbf{System}}\\
        \midrule
        $\mathcal{N}$ & Set of participating nodes \\
        $N$ & A node\\
        $n$ & Number of participating nodes\\
        $G(\mathcal{N}, E)$ & Aggregation graph \\
        $\delta$ & Degree of a node\\
        $\operatorname{View}(N)$ & Immediate neighbors of $N$\\
        $\operatorname{View}_2(N)$ & Second-degree neighbors of $N$\\
        $\operatorname{Comm}(N_{\ell}, N_i)$ & Set of nodes that $N_i$ establishes a mask with for masking the model sent to $N_\ell$\\
        \midrule
        \multicolumn{2}{c}{\textbf{Machine learning model}}\\
        \midrule
        $d$ & Number of parameters/model size\\
        $\mathcal{I}$ & Set of all possible indices/parameter positions\\
        $p$ & An index/parameter position\\
        $v$ & Non-masked model\\
        $v^{(p)}$ & Value of parameter of model $v$ at position $p$\\
        $v_{\ell \leftarrowvert i}$ & Masked model node $N_i$ sends to node $N_\ell$\\
        $M$ & A mask\\
        \midrule
        \multicolumn{2}{c}{\textbf{Collusion}}\\
        \midrule
        $s$ & Masking requirement\\
        $\mathcal{N}_T$ & Set of trusted nodes\\
        $\mathcal{N}_A$ & Set of colluding adversarial nodes\\
        $n_t$ & Number of trusted\\
        $k$ & Number of colluding adversarial nodes\\
        \midrule
        \multicolumn{2}{c}{\textbf{Sparsification}}\\
        \midrule
        $\alpha$ & Fraction of total parameters selected in local sparsification\\
        $\beta$ & Fraction of total parameters shared after discarding insufficiently masked parameters\\
        \bottomrule
    \end{tabular}
    \caption{List of symbols used in this paper}
    \label{tab:table_of_notation}
\end{table*}

\section{Postponed Proofs}

\subsection{Proof of \Cref{th:hbc_resilient}}\label[appendix]{app:proof:hbc_resilient}
\begin{proof}
Without the loss of generality, let us fix a node $N \in \mathcal{N}$ and consider an arbitrary parameter $p\in I$.
During \sys, node $N$ only receives masked models from its neighbors. If $N$ did not receive the parameter $p$ from any neighbor, it trivially does not learn anything about $p$.

If $N$ received the parameter $p$ from the neighbor $N_{a_0}$, it means that $N_{a_0}$ must have applied at least one mask to parameter $p$.
Under \sys, if $N_{a_0}$ has applied at least one mask, there must be a non-empty exhaustive set of the neighbors $\{N_{a_1}, \dots, N_{a_l}\}$ that share pairwise masks with $N_{a_0}$ for $p$.
Hence, each node of the set $\mathcal{A} = \{N_{a_0}, N_{a_1}, \dots, N_{a_l}\} $, with $l\ge 1$, must have intended to share $p$.
Furthermore, by the algorithm, every pair $\{i, j\}$ of nodes in $\mathcal{A}$ must have agreed on a random pairwise mask $M^{'(p)}_{ij} = -M^{'(p)}_{ji}$ for $p$.
Given that $v_{N_\ell}^{(p)}$ is a parameter value $p$ at $N_\ell$ before \sys and $w_{N_\ell}^{(p)}$ is the value received by $N$ from $N_\ell \in \mathcal{A}$ for parameter $p$, we know that $$w_{N_\ell}^{(p)} = v_{N_\ell}^{(p)} + \sum_{N_o \in \mathcal{A}\setminus\{N_\ell\}} M^{'(p)}_{N_\ell N_o}.$$

The information available to $N$ can be modelled by the system of linear equations:

\begin{align*}
    v_{N_{a_0}} + M_{N_{a_0}N_{a_1}} + M_{N_{a_0}N_{a_2}} + \dots + M_{N_{a_0}N_{a_l}} &= w_{N_{a_0}}\\
    v_{N_{a_1}} - M_{N_{a_0}N_{a_1}} + M_{N_{a_1}N_{a_2}} + \dots + M_{N_{a_1}N_{a_l}} &= w_{N_{a_1}}\\
    &\vdots \\
    v_{N_{a_l}} - M_{N_{a_0}N_{a_l}} - M_{N_{a_1}N_{a_l}} + \dots - M_{(l-1)l} &= w_{N_{a_l}}\\
\end{align*}

From this, w.l.o.g., we can consider a pair of nodes from $\mathcal{A}$: let them be $N_{a_0}$ and $ N_{a_1}$.
Next, we assume that all values in the system other than $v_{N_{a_0}}, v_{N_{a_1}}$ and $M_{N_{a_0}N_{a_1}}$ are uniquely determined.
Now we write down all equations where any of the unknown values appear:

\begin{align*}
    &v_{{N}_{a_0}} + M_{{N}_{a_0}{N}_{a_1}} + \underbrace{M_{N_{a_0} N_{a_2}} + \dots + M_{N_{a_0} N_{a_l}}}_{\hat{M}_{N_{a_0}}} = w_{N_{a_0}}\\
    \implies  &v_{N_{a_0}} + M_{N_{a_0}N_{a_1}} = \underbrace{w_{N_{a_0}} - \hat{M}_{N_{a_0}}}_{w^{'}_{N_{a_0}}}\\
    \cline{1-2}
    &v_{N_{a_1}} - M_{N_{a_0}N_{a_1}} + \underbrace{M_{N_{a_1}N_{a_2}} + \dots + M_{N_{a_1}N_{a_l}}}_{{\hat{M}}_{N_{a_2}}} = w_{N_{a_1}}\\
    \implies &  v_{N_{a_1}} - M_{N_{a_0}N_{a_1}} = \underbrace{w_{N_{a_1}} - \hat{M}_{N_{a_1}}}_{w^{'}_{N_{a_0}}}\\
    \cline{1-2}
\end{align*}

Thus, we obtain: 

\begin{align}
    &v_{N_{a_0}} + M_{N_{a_0}N_{a_1}} = w^{'}_{N_{a_0}}\nonumber\\
    &v_{N_{a_1}} - M_{N_{a_0}N_{a_1}} = w^{'}_{N_{a_1}}\nonumber
\end{align}

Here, we see a system of two equations with three unknowns from where no unknown can be uniquely determined.

Analogously, applying the same approach for any pair of nodes from $\mathcal{A}$ we see that none can be uniquely determined.
Therefore, honest-but-curious node $N$ cannot infer any $v_{N_{a_i}}$ for $\forall N_{a_i} \in \mathcal{A}$, \ie $N$ cannot infer an exact value of parameter at position $p$ for any neighbor that shares $p$ with $N$.

\end{proof}

\subsection{Proof of \Cref{lem:numberofmasks}}\label[appendix]{app:proof:numberofmasks}
\begin{proof}
    The set of indices selected by node $N \in \mathcal{N}$ in sparsification is denoted as $I_N \subseteq \mathcal{I}$.
    We look at the index $p$ and node $N_j$ that receives the masked models.
    Furthermore we define a set of neighbors of $N_j$ who have selected $p$ in sparsification as $Sel(N_j, p) = \{N \in \text{View}(N_j): p \in I_N\}$. 

    For every neighbor $N_i \in \text{View}(N_j)$ there are two cases to consider for index $p$:

    \textbf{Case 1.} $p \notin I_{N_i}$: If $N_i$ does not select $p$ for sharing, by the algorithm there trivially will not be any masks agreed between $N_i$ and another neighbor of $N_j$, making the number of mask in this case $0$.

    \textbf{Case 2.} $p \in I_{N_i}$:  If $N_i$ selects $p$ for sharing the total number of masks $N_i$ will agree upon with other neighbors of $N_j$ on $p$ is $\text{MaskCount}(N_i, N_j, p) = |\{N \in \text{View}(N_j)\setminus\{N_i\}: p \in I_{N_i \cap N}\}|$.
    Considering the precondition of the case, expression $p \in I_{N_i \cap N}$ can be transformed into $p \in I_N$ making $\text{MaskCount}(N_i, N_j, p) = |\{N \in \text{View}(N_j)\setminus\{N_i\}: p \in I_N\}|$.
    Because $p \in I_{N_i}$ then $N_i \in Sel(N_j, p)$, so the previous expression can also be written as:
    
    $\text{MaskCount}(N_i, N_j, p) = | Sel(N_j, p)\setminus\{N_i\}| = |Sel(N_j, p)|-1$.
    Repeating the same for all neighbors of $N_j$ we obtain the exact same value, hence proving the statement of the lemma.
\end{proof}

\subsection{Proof of \Cref{lem:graphreduction}}\label[appendix]{app:proof:graphreduction}
\begin{proof}
$i.$ If there are $s< k$ adversarial nodes in the graph, adding an adversarial node with degree $0$ will not change the number of masks that the trusted nodes applied to their model parameters.

$ii.$ If there is a trusted node with degree $0$, it does not contribute to the working of \sys or the number of masks that the other trusted nodes decide to apply. Therefore, we may remove them from the graph without compromising the privacy of any of the nodes. 

$iii.$ If there is any $N_0\in\mathcal{N}_A$ without having an adversary as its neighbor, for every $N\in\mathcal{N}_A$, adding $\{N_0, N\}$ to $E$ increments the number of masks added to the parameters of any trusted node connected to $N_0$ or $N$ by on (which will possibly be shared between $N$ and $N'$). However, more importantly, this does not reduce any masks already applied to the parameters of the trusted nodes, implying their privacy remains at least as it was when  $\{N_0, N\}$ was not in $E$. 

$iv.$ Let $\{N_i, N_j\} \in E$ for some $N_i, N_j \in \mathcal{N}_T$. Deleting this edge between $N_i$ and $N_j$ will not affect the masks that any trusted node applies to the parameters that they share with the adversarial nodes. We note, however, that deleting the edge between $N_i$ and $N_j$ will reduce the number of masks applied to the parameters that any other trusted node $N_{\ell}$ shares with $N_i$ or $N_j$. But, as $N_i$ or $N_j$ are not involved in any collusion (\ie, have no information about the masks used by any other node), as long as the parameters they receive from $N_{\ell}$ have at least one mask applied to them, the privacy of $N_{\ell}$ will be upheld.  
\end{proof}

\subsection{Proof of \Cref{th:collusions}}\label[appendix]{app:proof:collusions}
\begin{proof}
    By \Cref{lem:graphreduction}, it is sufficient to consider the case where we the trusted nodes have edges only with the adversarial nodes. Setting the number of adversarial nodes as $k$, as they collude in nature, we assume that they form a $k$ clique in the network. Letting $n_t$ be the number of trusted nodes present in the network, we proceed to show the result by inducting over $n_t$. 

    When $n_t=1$, denoting $N$ to be the only trusted node, as it is connected to an adversarial node $A$, $|\operatorname{Comm}(A, N)|=k-1$. Therefore, $N$ can have at most $k-1$ masks applied to any of its parameters (which are decided by \sys communicating with the members of $\operatorname{Comm}(A, N)$). As the masking requirement is set to $k$, $N$ will never end up sharing updates from any of its parameters with $A$, implying that its true value will never be revealed.

     Let the result hold for some $n_t=\lambda-1\geq 2$ and a new trusted node in the network. Let the trusted nodes be $N_1,\ldots N_{\lambda}$. Due to \Cref{lem:graphreduction}, we may ignore any edge between the trusted nodes and only consider the case where they have adversarial neighbors. Let $N_{i}$ have an edge with adversarial node $A_{i}$ for every $i\in\{1,\ldots,\lambda\}$. 
     
     \textbf{Case 1.} $A_{\lambda}\neq A_i$ for every $i\in\{1,\ldots,\lambda-1\}$: This essentially implies that $N_{\lambda}$ can have at most $k-1$ masks applied its values at any index and, therefore, will not share its model updates with the adversarial nodes and it will remain private.

     \textbf{Case 2.} $A_{\lambda}= A_i$ for some $i\in\{1,\ldots,\lambda-1\}$:  If $|\operatorname{Comm}(A_{\lambda}, N_{\lambda})|<k-1$, the maximum number of masks that $N_{\lambda}$ can apply to the value of its model update at any parameter will be less than its masking requirement of $k$ and, hence, will not share its model updates with the adversarial nodes and it will remain private. Otherwise, $N_{\lambda}$ and $N_i$ communicate and privately agree to mask their corresponding model updates at the indices in $I_{\lambda \cap i}$ with $M_{\lambda j}$ and  $M_{j\lambda}(=- M_{\lambda j})$, respectively. Thus, $\left\{v^{(p)}_{\lambda}\colon p\in I_{\lambda \cap i}\right\}$ will have at least one secure mask $M_{\lambda}$ applied to it for every $p\in I_{\lambda \cap i}$ which will not be available to the colluding adversarial nodes. Thus, the true values $\{v^{(p)}_{\lambda}\}$  will remain hidden from $A_{\lambda}$ (and, correspondingly, from every other adversarial node) due to the privately chosen mask between $N_{\lambda}$ and $N_{i}$.
\end{proof}

\subsection{Proof of \Cref{th:expected_parameters_shared_collusion}}\label[appendix]{app:proof:expected_parameters_shared_collusion}
\begin{proof}

Random subsampling ensures that the sparsification selects indices with uniform probability $\alpha$ and we assume that the sparsification carried out by each node is independent.

Without the loss of generality, let us consider the nodes $N,N'\in\mathcal{N}$ where $N$ is the sender and $N'$ is the receiver (i.e., $\{N,N'\}\in E$). With the masking requirement of $s$, we note that $N$ shares a parameter with $N'$ iff it has been sampled by at least $s$ other neighbors of $N'$. Therefore, the probability $\hat\pi^{(p)}_{N\rightarrow N'}$ of $N$ sharing a parameter $p\in I$ with $N'$ in one round of \sys is given by:

\begin{align*}
    \hat\pi^{(p)}_{N\rightarrow N'}&= \pi_i^{(p)} \left(1 - \prod_{j\in \operatorname{View}_2(N)} \left(1 - \pi_j^{(p)}\right)\right)\\
    &= \alpha\sum_{i=s}^{\delta-1} \binom{\delta-1}{i} \alpha^i (1-\alpha)^{\delta-1-i}\nonumber\\
    &=\sum_{i=s}^{\delta-1} \binom{\delta-1}{i} \alpha^{i+1} (1-\alpha)^{\delta-1-i}
\end{align*}

Hence, let $X_{N\rightarrow N'}$ be the random variable denoting the number of parameters that $N$ shares with $N'$ in one round with a masking requirement of $s$, and for every $p\in I$, let $X_{N\rightarrow N'}(p)=\mathbb{I}_{\{N\text{ shares } p \text{ with } N'\}}$ be the random variable denoting whether or not $N$ shares the parameter $p$ with $N'$. Therefore, noting that $X_{N\rightarrow N'}=\sum_{p\in I}X_{N\rightarrow N'}(p)$, we have:
\begin{align}
    \mathbb{E}\left[X_{N\rightarrow N'}\right]
    &=\mathbb{E}\left[\sum_{p\in I}X_{N\rightarrow N'}(p)\right]
    =\sum_{p\in I}\mathbb{E}\left[X_{N\rightarrow N'}(p)\right]\nonumber\\
    &=\sum_{p\in I}\mathbb{P}\left[\{N\text{ shares } p \text{ with } N'\}\right]\nonumber\\
    &=d \hat\pi^{(p)}_{N\rightarrow N'}.\nonumber
\end{align}

Thus, the expected fraction of the model parameters that $N$ shares with $N'$ is given by:
\begin{align}
    \beta\left(\alpha,\delta,s\right)&=\mathbb{E}\left[\frac{X_{N\rightarrow N'}}{d}\right]
    =\frac{\mathbb{E}\left[X_{N\rightarrow N'}\right]}{d}
    =\hat\pi^{(p)}_{N\rightarrow N'}\nonumber\\
    &=\sum_{i=s}^{\delta-1} \binom{\delta-1}{i} \alpha^{i+1} (1-\alpha)^{\delta-1-i}.\nonumber
\end{align}

\end{proof}

\subsection{Proof of \Cref{th:expected_parameters_shared}}\label[appendix]{app:proof:expected_parameters_shared}
\begin{proof}
    In the absence of collusion, the masking requirement of all the nodes is $1$. Thus, putting $s=1$ in \Cref{th:expected_parameters_shared_collusion}, we obtain:
    \begin{align}
        \beta\left(\alpha,\delta,1\right)
    &=\alpha\sum_{i=1}^{\delta-1} \binom{\delta-1}{i} \alpha^{i} (1-\alpha)^{\delta-1-i} \nonumber\\
    &=\alpha \left(1-(1-\alpha)^{\delta - 1}\right)\nonumber
    \end{align}
\end{proof} %
}{}

\end{document}